%% file: main_cvpr.tex
\documentclass[10pt,twocolumn,letterpaper]{article}

\usepackage{cvpr}              
\usepackage[accsupp]{axessibility}  

\usepackage{graphicx}
\usepackage{amssymb}
\usepackage{booktabs}
\usepackage{soul}
\usepackage[normalem]{ulem}
\newcommand{\stkout}[1]{\ifmmode\text{\sout{\ensuremath{#1}}}\else\sout{#1}\fi}
\input{configuration/math_commands}

\usepackage{amsmath}
\usepackage{amssymb}
\usepackage{mathtools}
\usepackage{amsthm}
\usepackage{bm}
\usepackage{thmtools,thm-restate}
\theoremstyle{plain}
\newtheorem{theorem}{Theorem}
\newtheorem{assumption}{Assumption}

\usepackage{graphicx, amsmath, amssymb, caption, subcaption, multirow, overpic, textpos}
\usepackage[dvipsnames]{xcolor}

\newcommand{\tablestyle}[2]{\setlength{\tabcolsep}{#1}\renewcommand{\arraystretch}{#2}\centering\footnotesize}
\newlength\savewidth\newcommand\shline{\noalign{\global\savewidth\arrayrulewidth
  \global\arrayrulewidth 1pt}\hline\noalign{\global\arrayrulewidth\savewidth}}

\usepackage[pagebackref,breaklinks,colorlinks]{hyperref}
\usepackage[capitalize]{cleveref}
\crefname{section}{Sec.}{Secs.}
\Crefname{section}{Section}{Sections}
\Crefname{table}{Table}{Tables}
\crefname{table}{Tab.}{Tabs.}

\usepackage{tikz}
\usetikzlibrary{bayesnet}
\usetikzlibrary{arrows}
\usetikzlibrary{backgrounds}
\newcommand{\lacl}{\color{brown}}
\newcommand{\blue}{\color{blue}}
\usepackage{tcolorbox}

\def\cvprPaperID{6956} 
\def\confName{CVPR}
\def\confYear{2023}

\begin{document}

\title{Understanding Masked Autoencoders via Hierarchical Latent Variable Models}


\author{
Lingjing Kong\thanks{Joint first author. $\dagger$ Joint senior author. }\, $^1$
\quad
Martin Q. Ma$^{* 1}$
\quad
Guangyi Chen$^{1, 2}$
\\
Eric P. Xing$^{1, 2}$
\quad
Yuejie Chi$^1$
\quad
Louis-Philippe Morency$^{\dagger 1}$
\quad
Kun Zhang$^{\dagger 1, 2}$
\\
\vspace{2mm}
{$^1$Carnegie Mellon University \quad $^2$Mohamed bin Zayed University of Artificial Intelligence}
}
\maketitle

\input{sections/abstract}

\input{sections/introduction}

\input{sections/theory}
\input{sections/experiments}
\input{sections/related_work}
\input{sections/conclusion}

\input{sections/acknowledgement}

{\small
\bibliographystyle{ieee_fullname}
\bibliography{cvpr}
}
\clearpage
\appendix

\input{sections/appendix}

\end{document}

%% file: configuration/math_commands.tex
\newcommand{\EEb}[2]{{\mathbb E}_{#1}\left[#2\right] }

\providecommand{\cc}{\mathbf{c}}
\providecommand{\dd}{\mathbf{d}}

\providecommand{\mm}{\mathbf{m}}

\providecommand{\pp}{\mathbf{p}}

\providecommand{\sss}{\mathbf{s}}

\providecommand{\vv}{\mathbf{v}}

\providecommand{\xx}{\mathbf{x}}

\providecommand{\zz}{\mathbf{z}}

\providecommand{\cC}{\mathcal{C}}

\providecommand{\cE}{\mathcal{E}}

\providecommand{\cP}{\mathcal{P}}

\providecommand{\cS}{\mathcal{S}}

\providecommand{\cV}{\mathcal{V}}
\providecommand{\cX}{\mathcal{X}}

\providecommand{\cZ}{\mathcal{Z}}

\providecommand{\R}{\mathbb{R}} 

\providecommand{\mE}{\mathbf{E}}

\providecommand{\mG}{\mathbf{G}}

\providecommand{\mV}{\mathbf{V}}

\providecommand{\mX}{\mathbf{X}}

\providecommand{\mZ}{\mathbf{Z}}

\providecommand{\ind}{\perp\!\!\!\!\perp}

\newcommand{\norm}[1]{\left\lVert#1\right\rVert}

\usepackage{algorithm}
\usepackage[noend]{algpseudocode}

\newcommand{\mycaptionof}[2]{\captionof{#1}{#2}}

\errorcontextlines\maxdimen

\makeatletter
    \newcommand*{\algrule}[1][\algorithmicindent]{\makebox[#1][l]{\hspace*{.5em}\thealgruleextra\vrule height \thealgruleheight depth \thealgruledepth}}%
\newcommand*{\thealgruleextra}{}
\newcommand*{\thealgruleheight}{.75\baselineskip}
\newcommand*{\thealgruledepth}{.25\baselineskip}

\newcount\ALG@printindent@tempcnta
\def\ALG@printindent{%
    \ifnum \theALG@nested>0
        \ifx\ALG@text\ALG@x@notext
        \else
            \unskip
            \addvspace{-1pt}
            \ALG@printindent@tempcnta=1
            \loop
                \algrule[\csname ALG@ind@\the\ALG@printindent@tempcnta\endcsname]%
                \advance \ALG@printindent@tempcnta 1
            \ifnum \ALG@printindent@tempcnta<\numexpr\theALG@nested+1\relax
            \repeat
        \fi
    \fi
    }%
\usepackage{etoolbox}
\patchcmd{\ALG@doentity}{\noindent\hskip\ALG@tlm}{\ALG@printindent}{}{\errmessage{failed to patch}}
\makeatother

\newbox\statebox
\newcommand{\myState}[1]{%
    \setbox\statebox=\vbox{#1}%
    \edef\thealgruleheight{\dimexpr \the\ht\statebox+1pt\relax}%
    \edef\thealgruledepth{\dimexpr \the\dp\statebox+1pt\relax}%
    \ifdim\thealgruleheight<.75\baselineskip
        \def\thealgruleheight{\dimexpr .75\baselineskip+1pt\relax}%
    \fi
    \ifdim\thealgruledepth<.25\baselineskip
        \def\thealgruledepth{\dimexpr .25\baselineskip+1pt\relax}%
    \fi
    \State #1%
    \def\thealgruleheight{\dimexpr .75\baselineskip+1pt\relax}%
    \def\thealgruledepth{\dimexpr .25\baselineskip+1pt\relax}%
}

%% file: sections/abstract.tex
\begin{abstract}
Masked autoencoder (MAE), a simple and effective self-supervised learning framework based on the reconstruction of masked image regions, has recently achieved prominent success in a variety of vision tasks. 
Despite the emergence of intriguing empirical observations on MAE, a theoretically principled understanding is still lacking. 
In this work, we formally characterize and justify existing empirical insights and provide theoretical guarantees of MAE.
We formulate the underlying data-generating process as a hierarchical latent variable model and show that under reasonable assumptions, MAE provably identifies a set of latent variables in the hierarchical model, explaining why MAE can extract high-level information from pixels.
Further, we show how key hyperparameters in MAE (the masking ratio and the patch size) determine which true latent variables to be recovered, therefore influencing the level of semantic information in the representation. 
Specifically, extremely large or small masking ratios inevitably lead to low-level representations.
Our theory offers coherent explanations of existing empirical observations and provides insights for potential empirical improvements and fundamental limitations of the masking-reconstruction paradigm.
We conduct extensive experiments to validate our theoretical insights. 

\end{abstract}

%% file: sections/introduction.tex
\vspace{-2em}
\section{Introduction}

Self-supervised learning (SSL) has achieved tremendous success in learning transferable representations without labels, showing strong results in a variety of downstream tasks \cite{devlin2018bert, chen2020simple, he2020momentum, chen2021empirical,radford2021learning}. 
As a major SSL paradigm, masked image modeling (MIM)~\cite{baevski2022data2vec,anandkumar2013learning,chen2020generative,xie2022simmim,he2021masked,zhou2021ibot,bao2021beit,chen2022context,li2021mst} performs the reconstruction of purposely masked image pixels as the pretraining task.  
Among MIM methods, masked autoencoding (MAE)~\cite{he2021masked} has gained significant traction due to its computational efficiency and state-of-the-art performance in a wide range of downstream tasks.



Empirical observations from previous work reveal various intriguing properties of MAE.
In particular, aggressive masking has been shown critical to downstream task performances~\cite{he2021masked,xie2022simmim,wu2022extreme,hu2022exploring}. 
It is conjectured that such masking forces the model to learn meaningful \textit{high-level} semantic understanding of the objects and scenes rather than the \textit{low-level} information such as texture. 
However, it remains largely unclear whether such intuitions are sound in principle.
Theoretically verifying and characterizing these empirical insights would not only grant a certificate to the current approaches but would also offer theoretical insights for algorithmic advancements.

\begin{figure}[t]
    \centering
    \includegraphics[width=0.9\columnwidth]{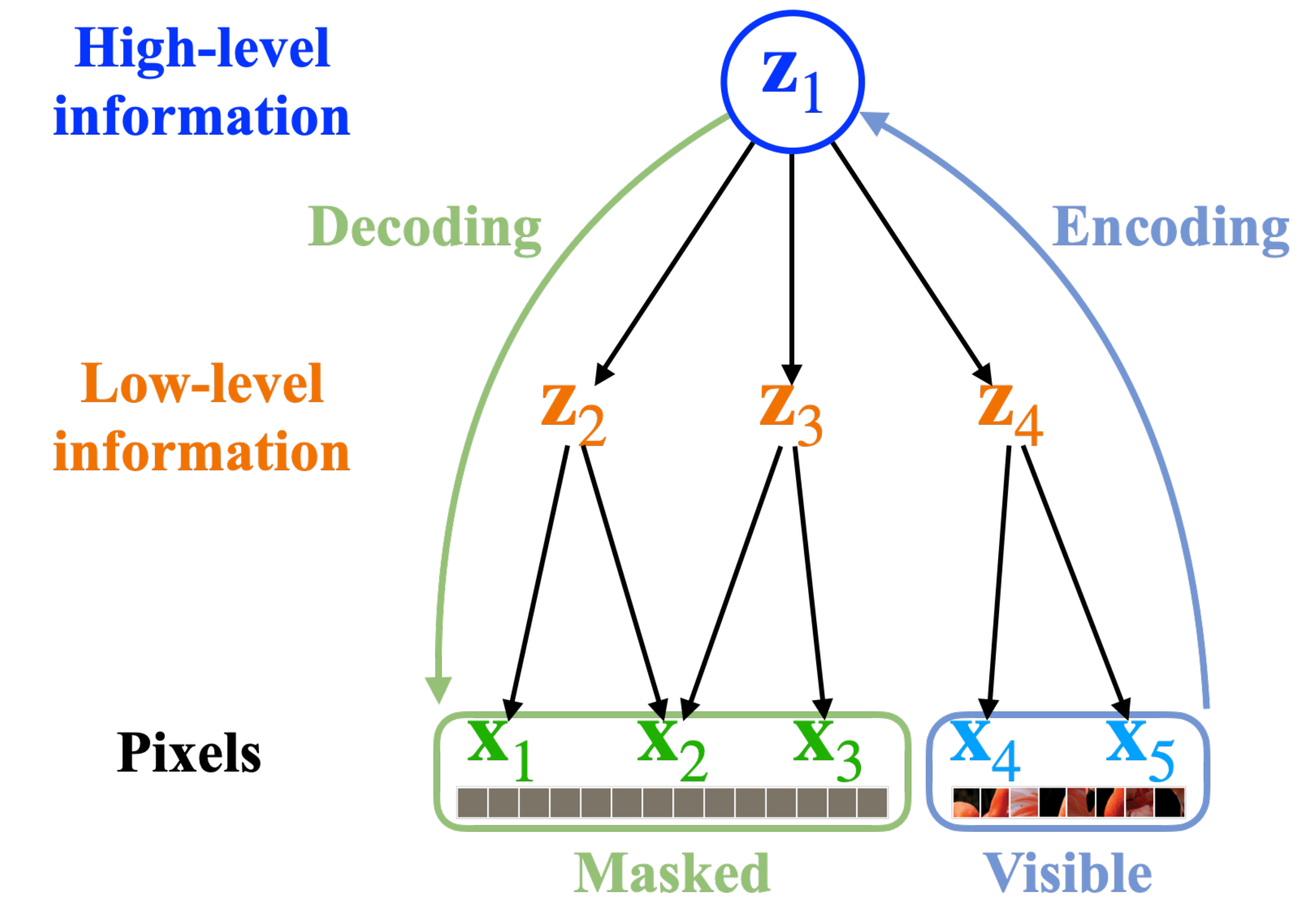}
    \caption{\footnotesize
    \textbf{Masking-reconstruction under a hierarchical generating process.}
    In a hierarchical data-generating process, high-level latent variables (e.g., \ $\zz_{1}$) represent high-level information such as semantics, and low-level latent variables (e.g., \ $[\zz_{2}, \zz_{3}, \zz_{4}]$) represent low-level information such as texture.
    We show that through proper masking, MAE learns to recover high-level latent variables with identifiability guarantees. 
    }
    \label{fig:intuition}
    \vspace{-2em}
\end{figure}


In this work, we establish a principled yet intuitive framework for understanding MAE and providing identifiability guarantees.
Concretely, we first formulate the underlying data-generating process as a hierarchical latent variable model (Figure~\ref{fig:intuition}), with high-level variables corresponding to abstract and semantic information like classes, and low-level variables corresponding to elaborate and granular information like texture.
Such latent variable models have been studied in causal discovery~\cite{xie2022identification,huang2022latent}. In \cite{sabour2017dynamic,hinton2021represent}, it is hypothesized that complex data, such as images, follow a hierarchical latent structure.

Stemming from this formulation, we show that under reasonable assumptions, MAE can recover a subset of the true latent variables within the hierarchy, where the levels of the learned latent variables are explicitly determined by how masking is performed.
Our theoretical framework not only unifies existing empirical observations coherently but also gives rise to insights for potential empirical improvements and fundamental limitations of MAE.
Our theory improves the existing nonlinear identifiability results~\cite{von2021self,lyu2021latent} and can be of independent interest. 


Empirically, we deduce several insights from our theoretical results and verify them with experiments.
Unlike common belief, MAE trained with extremely high masking ratios (e.g., $90\%$) captures low-level information, similar to models trained with extremely low ratios (e.g., $10\%$).
Our results suggest that learning high-level semantic information is only possible in the non-extreme masking regime. 
We also discuss masking designs that can potentially improve current empirical performance.

\paragraph{Contributions.} We highlight the following contributions:
\begin{itemize}
    \item We formulate the underlying data-generating process as a hierarchical latent variable model. Under such a formulation, we provide a theoretical guarantee for MAE by showing that it can recover true latent variables in the hierarchical model. 
    \item Based on our theoretical results, we establish the connection between masking hyperparameters (i.e., masking ratios and patch sizes) and the learned representation and discuss potential improvements and inherent limitations of MAE.
    \item We validate our theoretical insights with extensive experimental results. We illustrate how the semantic level of the learned representation varies with the aggressiveness of the masking strategy. Interestingly, representations learned under overly aggressive masking (e.g.,ß 90\% masking ratio) exhibit similar properties to their counterparts learned with overly conservative masking (e.g., 10\% masking ratio).  
\end{itemize}

%% file: sections/theory.tex
\section{Theoretical Understanding}
\label{sec:theory}

\subsection{A Hierarchical Data-generating Process}
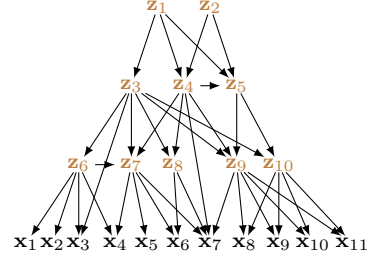
\begin{figure}[htp!]
\centering
	\begin{tikzpicture}[scale=.7, line width=0.4pt, inner sep=0.1mm, shorten >=.1pt, shorten <=.1pt]
		\draw (-1.5, 2.5) node(L1)  {{\footnotesize\lacl\,{$\zz_1$}\,}};
		\draw (-0.5, 2.5) node(L2)  {{\footnotesize\lacl\,{$\zz_2$}\,}};
		\draw (-2, 1) node(L3) {{\footnotesize\lacl\,{$\zz_3$}\,}};
		\draw (-1, 1) node(L4)  {{\footnotesize\lacl\,$\zz_4$\,}};
		\draw (0, 1) node(L5)  {{\footnotesize\lacl\,$\zz_5$\,}};
		\draw (-3, -0.5) node(L6)  {{\footnotesize\lacl\,$\zz_6$\,}};
		\draw (-2, -0.5) node(L7)  {{\footnotesize\lacl\,$\zz_7$\,}};
		\draw (-1.2, -0.5) node(L8)  {{\footnotesize\lacl\,$\zz_8$\,}};
		\draw (0, -0.5) node(L9)  {{\footnotesize\lacl\,$\zz_9$\,}};
		\draw (0.8, -0.5) node(L10)  {{\footnotesize\lacl\,$\zz_{10}$\,}};
		\draw (-4, -2) node(X1)  {{\footnotesize\,$\xx_1$\,}};
		\draw (-3.5, -2) node(X2)  {{\footnotesize\,$\xx_2$\,}};
		\draw (-3, -2) node(X3)  {{\footnotesize\,$\xx_3$\,}};
		\draw (-2.3, -2) node(X4)  {{\footnotesize\,$\xx_4$\,}};
		\draw (-1.7, -2) node(X5)  {{\footnotesize\,$\xx_5$\,}};
		\draw (-1.1, -2) node(X6)  {{\footnotesize\,$\xx_6$\,}};
		\draw (-0.5, -2) node(X7)  {{\footnotesize\,$\xx_7$\,}};
		\draw (0.15, -2) node(X8)  {{\footnotesize\,$\xx_8$\,}};
		\draw (0.8, -2) node(X9)  {{\footnotesize\,$\xx_9$\,}};
		\draw (1.45, -2) node(X10)  {{\footnotesize\,$\xx_{10}$\,}};
		\draw (2.2, -2) node(X11)  {{\footnotesize\,$\xx_{11}$\,}};
	   \draw[-latex] (L1) -- (L3);
	   \draw[-latex] (L1) -- (L4);
	   \draw[-latex] (L1) -- (L5);
	   \draw[-latex] (L2) -- (L4);
	   \draw[-latex] (L2) -- (L5);
	   \draw[-latex] (L3) -- (L6);
	   \draw[-latex] (L3) -- (L7);
	   \draw[-latex] (L3) -- (L8);
	   \draw[-latex] (L3) -- (L9);
	   \draw[-latex] (L3) -- (L10);
	   \draw[-latex] (L3) -- (X3);
	   
	   \draw[-latex] (L4) -- (L5);
	   \draw[-latex] (L4) -- (L7);
	   \draw[-latex] (L4) -- (L8);
	   \draw[-latex] (L4) -- (L9);
	   \draw[-latex] (L4) -- (X7);
	   \draw[-latex] (L5) -- (L9);
	   \draw[-latex] (L5) -- (L10);
	   \draw[-latex] (L6) -- (X1);
	   \draw[-latex] (L6) -- (X2);
	   \draw[-latex] (L6) -- (X3);
	   \draw[-latex] (L6) -- (X4);
	   \draw[-latex] (L6) -- (L7);
	   \draw[-latex] (L7) -- (X4);
	   \draw[-latex] (L7) -- (X5);
	   \draw[-latex] (L7) -- (X6);
	   \draw[-latex] (L7) -- (X7);
	   \draw[-latex] (L8) -- (X6);
	   \draw[-latex] (L8) -- (X7);
	   \draw[-latex] (L9) -- (X7);
	   \draw[-latex] (L9) -- (X8);
	   \draw[-latex] (L9) -- (X9);
	   \draw[-latex] (L9) -- (X10);
	   \draw[-latex] (L9) -- (X11);
	   \draw[-latex] (L10) -- (X8);
	   \draw[-latex] (L10) -- (X9);
	   \draw[-latex] (L10) -- (X10);
	   \draw[-latex] (L10) -- (X11);
	   \label{fig:hierarchical-a}
	\end{tikzpicture}
\vspace{-.4em}
\caption{\textbf{A hierarchical data-generating process.} $\zz$ represents the latent variables and $\xx$ stands for the observable variables (i.e.\ image pixels). The hierarchical model is generic and is capable of modeling arbitrary DAGs in the latent space.}
\vspace{-1.5em}

\label{fig:hierarchical}
\end{figure}

Images, despite their high dimensionality, are well structured -- there is a multitude of statistical dependencies among pixels determined by their relative distances and visual semantics. 
For instance, pixels in close proximity are often highly dependent, whereas pixels far apart typically share less information. 
There has been a plethora of work adopting this intuition for vision tasks such as image generation~\cite{vahdat2020nvae,maaloe2019biva,zhao2017learning}.
Similar insights are also addressed in attempts to learn a part-whole image representation~\cite{sabour2017dynamic,hinton2021represent}.

In this work, we formulate such an underlying structure of images with a hierarchical data-generating process~\cite{anandkumar2013learning,xie2022identification,huang2022latent} (Figure~\ref{fig:hierarchical}).
Under this formulation, we reveal the underpinning principle of MAE and provide identifiability guarantees. In particular, we show that through masking-reconstruction, MAE learns the long-range statistical dependencies within the image, which renders it capable of extracting high-level semantic representations.

Formally, the generating process is defined with a graph structure $\mG: = ( \mV, \mE )$ where $\mE$ is the set of all directed edges and $ \mV := (\mX, \mZ)$ comprises all observable variables $\mX := \{\xx_{1}, \dots, \xx_{m} \}$ (i.e.,\ all pixels) and all latent variables $ \mZ:= \{\zz_{1}, \dots, \zz_{n} \}$. 
Each variable $\xx_{i}$ or $\zz_{j}$ represents a multidimensional vector. 
\footnote{
    In high-dimensional data like images, there is a larger degree of information redundancy, e.g., neighboring pixels. Thus, it is sensible to lump one-dimensional variables into vectors. 
}
The hierarchical latent structure $\mG$ fulfills the following assumption:
\begin{assumption} \label{assumption:data_generating_process}
(Data-generating process):  There is no direct edge between any two observables: $ \forall \xx_{i}, \xx_{j} \in \mX$ , $ (\xx_{i}, \xx_{j}) \notin \mE $ and $ (\xx_{j}, \xx_{i}) \notin \mE $. Each variable is generated by its parents in a directed acyclic graph (DAG) according to:
    \begin{align} \label{eq:sem}
    \begin{split}
        \zz_{i} &= g_{\zz_{i}} ( \text{Pa}(\zz_{i}),
        \bm\varepsilon_{i}), \\
         \xx_{j} &= g_{\xx_{j}} ( \text{Pa}(\xx_{j}),
        \bm\varepsilon_{j}),
    \end{split}
    \end{align}
    where $ g_{\zz_{i}} $ and $g_{\xx_{j}}$ are invertible functions, $\bm\varepsilon_{i}$ denotes exogenous random variables, and $\text{Pa}(\cdot)$ denotes the parents of a certain node. 
\end{assumption}

The invertible data-generating-module assumption ($g_{i} $ and $g_{j}$ being invertible) is adopted from prior work identifying latent variables in deep generative models \cite{locatello2020weakly, von2021self}. We make the following remarks on the hierarchical generating process. 
First, we note that we impose minimal constraints on the graph structure among the latent variables (i.e., the connectivity among latent variables $\zz$); therefore, the hierarchical model class is generic and encompasses all possible DAG structures over latent variables (Figure~\ref{fig:hierarchical}). 
Next, we interpret the latent variables $\zz$ as information related to semantic/content information, such as the shape and contour in the image, whereas the exogenous variables $\bm\varepsilon$ injected in each layer represent nuanced information, such as the texture and contrast of the image. 
Each structural function $g_{i}$ mixes the two sources of information and generates a more low-level variable until pixels $\xx$.
Lastly, for the upcoming theoretical results, as long as the data-generating process conforms to the hierarchical graph assumption, our theory holds, and the insights do not rely on the knowledge of a specific graph structure.


\subsection{Masked Autoencoders}
As a canonical method of masking-reconstruction learning, MAE~\cite{he2021masked} randomly masks a subset of pixel patches in the original image and then reconstructs the masked patches from the encoded representation of the visible part. 
More formally, we formulate the MAE training as follows.

\textbf{Mask sampling}: random masks $\mm$ are sampled from a distribution $p_{\mm}$ which is parameterized by the masking ratio $ r $ (i.e., the ratio between the number of masked pixels and the number of all pixels) and patch size $s$ (i.e., the size of the minimal masking unit).

\textbf{MAE encoding}: $E_{\mm^{c}} (\xx_{\mm^{c}})$ maps the unmasked part $ \xx_{\mm^{c}} $ to a latent representation $ \hat{\cc} $ \footnote{
        To avoid notation cluttering, we adopt $ \hat{\cdot} $ to distinguish the estimated variables from the true ones in the generating process.
    }, where $ \mm^{c} $ denotes the complement of the mask index set $\mm$ and is passed to the encoder as positional embeddings to indicate the positions of the visible patches.
    
\textbf{MAE decoding}: $D_{\mm}(\hat{\cc}, \hat{\sss}_{\mm})$ reconstructs the masked image $\xx_{\mm}$ from the estimated latent variable $ \hat{\cc} $ (i.e., the encoder output), and the auxiliary information $\hat{\sss}_{\mm}$ embodying positional embeddings and \texttt{[MASK]} token which are fed to the decoder in MAE. 
    Although $\hat{\sss}_{\mm}$ is deterministic in MAE implementation, we view it as a random variable in our analysis.

With the notation above, the MAE training objective can be expressed as follows:
{
    \small
    \begin{align} \label{eq:mae_obj}
    L ( E, D ) := \EEb{\mm, \xx, \hat{\sss}_{\mm}}{
        \norm{ D_{\mm} \left( E_{\mm^{c}} (\xx_{\mm^{c}}), \hat{\sss}_{\mm} \right) - \xx_{\mm} }^{2}
    }. 
    \end{align}
}



\begin{figure} 
    \centering
    \tikz{
        \node[obs] (x1) {$\xx_{\mm}$};%
        \node[obs,right=1cm of x1] (x2) {$\xx_{\mm^{c}}$};%
        \node[latent,above=1cm of x1,xshift=-1.5cm] (s1) {$\sss_{\mm}$}; %
        \node[latent,above=1cm of x2,xshift=1.5cm] (s2) {$\sss_{\mm^{c}}$}; 
        \node[latent,above=1cm of x1,xshift=1cm] (c) {$\cc$}; %
        \edge[-,dashed] {c} {s2}
        \edge {s1} {x1}
        \edge {s2} {x2}
        \edge {c} {x1,x2} 
}
    \caption{\textbf{Information sharing latent models}. Here, $\xx_{\mm}$ and $ \xx_{\mm^{c}} $ denote the masked part and the visible part of the image $\xx$, respectively. $\cc$ stands for the maximally shared information between $\xx_{\mm}$  and $ \xx_{\mm^{c}} $. $\sss_{\mm}$ and $ \sss_{\mm^{c}} $ refer to the information specific to $\xx_{\mm}$  and $ \xx_{\mm^{c}} $ respectively. The dashed line indicates the potential existence of statistical dependence.}
    \label{fig:cs_model}
\end{figure}

\subsection{Identifiability Theory}

Building upon the formalization above, we show in Theorem~\ref{thm:locating_c} that each random mask $\mm$ would induce a specific (sub)set of latent variables that fully captures the statistical dependency between the masked part and the visible part. We denote this relationship as $\cc \subset \mZ$ where $\cc$ is the subset of the latent variable set $\mZ$. 

\begin{restatable}{theorem}{locatec} \label{thm:locating_c}
    (Locating the shared information $\cc$): In a hierarchical latent variable structure $\mG$, for each specific mask $ \mm $, there exists a corresponding minimal set of latent variables $\cc$ such that the generating process of $\xx$ can be expressed as in Figure~\ref{fig:cs_model} where the following conditions are satisfied:
    \begin{enumerate}
        \item $\xx_{\mm} = g_{\xx_{\mm}} ( \cc, \sss_{\mm} ) $ and $ \xx_{\mm^{c}} = g_{\xx_{\mm^{c}}} ( \cc, \sss_{\mm^{c}} ) $ where both $ g_{\xx_{\mm}} $ and $g_{\xx_{\mm^{c}}}$ are invertible;
        \item $ \sss_{\mm} \ind (\cc, \sss_{\mm^{c}}) $;
        \item $ \cc $ is minimal: $ \forall \cc' \subset \mZ $ such that $ \text{dim}(\cc') < \text{dim} (\cc) $, $\cc'$ cannot satisfy the two conditions above.
    \end{enumerate}
    Such $\cc$ and the corresponding $\sss_{\mm}$ are unique and can be located from the hierarchical structure by executing Algorithm~\ref{alg:locate_c}. Furthermore, $ \sss_{\mm^{c}} $ can be found through Algorithm~\ref{alg:locate_smc}.
\end{restatable}
The proof, Algorithm~\ref{alg:locate_c}, and Algorithm~\ref{alg:locate_smc} can be found in Appendix~\ref{sec:proof_theorem_graph}. We note that although the minimal $\cc$ and its corresponding $\sss_{\mm}$ are unique for a given mask $\mm$, there is no unique $\sss_{\mm^{c}}$ in general. Algorithm~\ref{alg:locate_smc} returns one such instance.

Theorem~\ref{thm:locating_c} states that for each mask $\mm$, there exists a corresponding $\cc$ that represents all the information contained in the visible part $ \xx_{\mm^{c}} $ that is conducive to reconstructing the masked part $ \xx_{\mm} $. 
Algorithm~\ref{alg:locate_c} can locate such $\cc$ in the hierarchy and directly characterizes the impact of masking on the property of $\cc$.

Next, in Theorem~\ref{thm:mae_identifiability}, we show that MAE learning objective (Equation~\ref{eq:mae_obj}) estimates $\cc$ specified in Theorem~\ref{thm:locating_c}, and MAE attains a form of identifiability of $\cc$. We first lay out the assumptions:
\begin{assumption} \label{assumption:mae_model} 
(MAE model): For any mask $\mm$, the MAE decoder $ D_{\mm}(\hat{\cc}, \hat{\sss}_{\mm}) $ has a non-singular Jacobian matrix almost anywhere, and there exists an invertible function $ \tilde{g}_{\mm^{c}} (\cdot) $ such that MAE encoder $ E_{\mm^{c}}( \cdot ) = [\tilde{g}_{\mm^{c}}^{-1} (\cdot) ]_{1:d_{c}} $ where $[\cdot]_{1:d_{c}}$ denotes the dimensions corresponding to $\cc$.
Moreover, $(D_{\mm}, \tilde{g}_{\mm^{c}})$ forms an invertible mapping between $ (\hat{\cc}, \hat{\sss}_{\mm}, \hat{\sss}_{\mm^{c}}) $ and $ (\xx_{\mm}, \xx_{\mm^{c}}) $
\end{assumption}

Next, we show MAE identifies the shared information $\cc$:

\begin{restatable}{theorem}{globalidentifiability} \label{thm:mae_identifiability}
    (Identifiability of $\cc$): 
    For each mask $\mm$, given the dimensions $ (d_{\cc}, d_{\sss_{\mm}})$ the encoder function $E_{\mm^{c}}(\cdot)$ recovers all information of $\cc$ located in Theorem~\ref{alg:locate_c}, i.e., there exists a one-to-one mapping $h$, s.t., $ h(\cc) = \hat{\cc} $.
\end{restatable}
In the following, we discuss our assumptions and results. The proof can be found in Appendix~\ref{sec:identifiability_proof}. 
\paragraph{Assumption interpretation.} 
Assumption~\ref{assumption:data_generating_process} follows prior work identifying latent variables in deep generative models~\cite{locatello2020weakly,von2021self} to ensure that latent variables are recoverable from pixels.
Assumption~\ref{assumption:mae_model} requires the MAE encoder $E_{\mm^{c}}$ to be part of an invertible function output -- this is mild and allows the encoder to be more flexible than invertible functions.
The decoder $D_{\mm} (\hat{\cc}, \hat{\sss}_{\mm})$ is assumed to be locally invertible in $\hat{\cc}$ almost surely, allowing for a broader class than invertible functions, e.g., nondegenerate polynomials.
The joint invertibility of $(D_{\mm}, \tilde{g}_{\mm^{c}})$ ensures no information loss during the estimation process.

\paragraph{How does MAE work?}
Theorem~\ref{thm:mae_identifiability} states that the MAE objective (Equation~\ref{eq:mae_obj}) essentially serves to estimate the shared variable $\cc$ and is able to restore all information in $\cc$.
Therefore, the efficacy of MAE stems from its ability to extract high-level semantic representations from low-level features like image pixels.
Moreover, our theory indicates the possibility of fully identifying a latent hierarchical structure via properly designed self-supervised objectives, opening up research avenues for future work.

\textbf{Takeaway}: \textit{\ul{MAE provably recovers high-level representations from low-level features like pixels.}}

\label{para:masking_mae_theory}
\paragraph{How does masking influence the learned representation?}
Theorem~\ref{thm:locating_c} establishes a direct connection between the mask $\mm$ and the shared information $\cc$, which is further connected to the MAE estimate $\hat{\cc}$ in Theorem~\ref{thm:mae_identifiability}.
We can observe that conservative masking with overly small masking ratios and masking patch sizes inevitably leads to low-level latent variables. 
To see this, in Figure~\ref{subfig:low_masking_intensity}, the mask is not large enough to cover all observable descendants of a desirable high-level variable $\zz_{1}$, thus following Algorithm~\ref{alg:locate_c} a low-level variable $\zz_{3}$ will mix in $\hat{\cc}$, preventing the model from learning $\zz_{1}$.
This insight highlights the necessity of nontrivial masking ratios and patch sizes and resonates with the empirical observations in \cite{he2021masked,hu2022exploring,xie2022simmim}.

\begin{figure}[htp!]
    \centering
    \begin{subfigure}{0.32\columnwidth}
    \begin{tikzpicture}[scale=.6, line width=0.4pt, inner sep=0.1mm, shorten >=.1pt, shorten <=.1pt]
		\draw (-2.3, 2.5) node(L1)  {{\footnotesize\lacl\,{$\zz_1$}\,}};
		\draw (-1, 2.5) node(L2)  {{\footnotesize\lacl\,{$\zz_2$}\,}};
		\draw (-3, 1) node(L3) {{\footnotesize\blue\,{$\zz_3$}\,}};
		\draw (-2, 1) node(L4)  {{\footnotesize\lacl\,$\zz_4$\,}};
		\draw (-1, 1) node(L5)  {{\footnotesize\lacl\,$\zz_5$\,}};
		\draw (0, 1) node(L6)  {{\footnotesize\lacl\,$\zz_6$\,}};
		\draw (-3.3, -0.5) node(X1)  {{\footnotesize\,$\stkout{\xx_1}$\,}};
		\draw (-2.6, -0.5) node(X2)  {{\footnotesize\,$\xx_2$\,}};
		\draw (-1.8, -0.5) node(X3)  {{\footnotesize\,$\xx_3$\,}};
		\draw (-1.1, -0.5) node(X4)  {{\footnotesize\,$\xx_4$\,}};
		\draw (-.4, -0.5) node(X5)  {{\footnotesize\,$\xx_5$\,}};
		\draw (0.3, -0.5) node(X6)  {{\footnotesize\,$\xx_{6}$\,}};
	   \draw[-latex] (L1) -- (L3);
	   \draw[-latex] (L1) -- (L4);
	   \draw[-latex] (L2) -- (L4);
	   \draw[-latex] (L2) -- (L5);
	   \draw[-latex] (L2) -- (L6);
	   \draw[-latex] (L3) -- (X1);
	   \draw[-latex] (L3) -- (X2);
	   \draw[-latex] (L3) -- (X3);
	   
	   \draw[-latex] (L4) -- (X2);
	   \draw[-latex] (L4) -- (X3);
	   \draw[-latex] (L5) -- (X4);
	   \draw[-latex] (L5) -- (X5);
	   \draw[-latex] (L6) -- (X4);
	   \draw[-latex] (L6) -- (X5);
	   \draw[-latex] (L6) -- (X6);
	\end{tikzpicture}
	\caption{\scriptsize\textbf{Conservative mask}}
	\label{subfig:low_masking_intensity}
    \end{subfigure}
    \begin{subfigure}{0.32\columnwidth}
    \begin{tikzpicture}[scale=.6, line width=0.4pt, inner sep=0.1mm, shorten >=.1pt, shorten <=.1pt]
		\draw (-2.3, 2.5) node(L1)  {{\footnotesize\lacl\,{$\zz_1$}\,}};
		\draw (-1, 2.5) node(L2)  {{\footnotesize\lacl\,{$\zz_2$}\,}};
		\draw (-3, 1) node(L3) {{\footnotesize\lacl\,{$\zz_3$}\,}};
		\draw (-2, 1) node(L4)  {{\footnotesize\lacl\,$\zz_4$\,}};
		\draw (-1, 1) node(L5)  {{\footnotesize\lacl\,$\zz_5$\,}};
		\draw (0, 1) node(L6)  {{\footnotesize\blue\,$\zz_6$\,}};
		\draw (-3.3, -0.5) node(X1)  {{\footnotesize\,$\stkout{\xx_1}$\,}};
		\draw (-2.6, -0.5) node(X2)  {{\footnotesize\,$\stkout{\xx_2}$\,}};
		\draw (-1.8, -0.5) node(X3)  {{\footnotesize\,$\stkout{\xx_3}$\,}};
		\draw (-1.1, -0.5) node(X4)  {{\footnotesize\,$\stkout{\xx_4}$\,}};
		\draw (-.4, -0.5) node(X5)  {{\footnotesize\,$\stkout{\xx_5}$\,}};
		\draw (0.3, -0.5) node(X6)  {{\footnotesize\,$\xx_{6}$\,}};
	   \draw[-latex] (L1) -- (L3);
	   \draw[-latex] (L1) -- (L4);
	   \draw[-latex] (L2) -- (L4);
	   \draw[-latex] (L2) -- (L5);
	   \draw[-latex] (L2) -- (L6);
	   \draw[-latex] (L3) -- (X1);
	   \draw[-latex] (L3) -- (X2);
	   \draw[-latex] (L3) -- (X3);
	   
	   \draw[-latex] (L4) -- (X2);
	   \draw[-latex] (L4) -- (X3);
	   \draw[-latex] (L5) -- (X4);
	   \draw[-latex] (L5) -- (X5);
	   \draw[-latex] (L6) -- (X4);
	   \draw[-latex] (L6) -- (X5);
	   \draw[-latex] (L6) -- (X6);
	\end{tikzpicture}
	\caption{\scriptsize\textbf{Aggressive mask}}
	\label{subfig:high_masking_intensity}
    \end{subfigure}
    \begin{subfigure}{0.32\columnwidth}
    \begin{tikzpicture}[scale=.6, line width=0.4pt, inner sep=0.1mm, shorten >=.1pt, shorten <=.1pt]
		\draw (-2.3, 2.5) node(L1)  {{\footnotesize\lacl\,{$\zz_1$}\,}};
		\draw (-1, 2.5) node(L2)  {{\footnotesize\blue\,{$\zz_2$}\,}};
		\draw (-3, 1) node(L3) {{\footnotesize\lacl\,{$\zz_3$}\,}};
		\draw (-2, 1) node(L4)  {{\footnotesize\lacl\,$\zz_4$\,}};
		\draw (-1, 1) node(L5)  {{\footnotesize\lacl\,$\zz_5$\,}};
		\draw (0, 1) node(L6)  {{\footnotesize\lacl\,$\zz_6$\,}};
		\draw (-3.3, -0.5) node(X1)  {{\footnotesize\,$\stkout{\xx_1}$\,}};
		\draw (-2.6, -0.5) node(X2)  {{\footnotesize\,$\stkout{\xx_2}$\,}};
		\draw (-1.8, -0.5) node(X3)  {{\footnotesize\,$\stkout{\xx_3}$\,}};
		\draw (-1.1, -0.5) node(X4)  {{\footnotesize\,$\xx_4$\,}};
		\draw (-.4, -0.5) node(X5)  {{\footnotesize\,$\xx_5$\,}};
		\draw (0.3, -0.5) node(X6)  {{\footnotesize\,$\xx_{6}$\,}};
	   \draw[-latex] (L1) -- (L3);
	   \draw[-latex] (L1) -- (L4);
	   \draw[-latex] (L2) -- (L4);
	   \draw[-latex] (L2) -- (L5);
	   \draw[-latex] (L2) -- (L6);
	   \draw[-latex] (L3) -- (X1);
	   \draw[-latex] (L3) -- (X2);
	   \draw[-latex] (L3) -- (X3);
	   
	   \draw[-latex] (L4) -- (X2);
	   \draw[-latex] (L4) -- (X3);
	   \draw[-latex] (L5) -- (X4);
	   \draw[-latex] (L5) -- (X5);
	   \draw[-latex] (L6) -- (X4);
	   \draw[-latex] (L6) -- (X5);
	   \draw[-latex] (L6) -- (X6);
	\end{tikzpicture}
	\caption{\scriptsize\textbf{Ideal mask}}
	\label{subfig:ideal_masking_intensity}
    \end{subfigure}
\caption{\textbf{The impact of masking on the learned representation.} We label the masked pixels with $\stkout{\xx}$. We locate the MAE learned latent variables with Algorithm~\ref{alg:locate_c} and label them with blue. We can observe that extremely low (left) and high (middle) masking intensities lead to low-level representations, whereas the desirable masking intensity that yields a high-level representation lies in the intermediate masking aggressiveness.}
\vspace{-1em}
\label{fig:illustrations_masking_effects}
\end{figure}
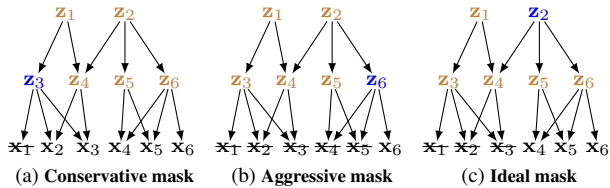

Surprisingly, the above reasoning can be applied to the case with extremely aggressive masking: in Figure~\ref{subfig:high_masking_intensity} low-level latent variables $\zz_{6}$ will be learned by MAE when the visible part is too small to cover all observable descendants of a desirable high-level variable $\zz_{2}$.
Thus, the learned representation does not become monotonically more high-level with increasing masking aggressiveness -- overly aggressive masking also gives rise to low-level representations.
This insight echoes the empirical finding in \cite{xie2022simmim,wu2022extreme} where the extremely large masking degrades the performance of high-level downstream tasks like classification~\cite{xie2022simmim} but yields relatively low-level representations like the object locations/scales in the image~\cite{wu2022extreme}.
In Section~\ref{sec:experiments}, we present empirical evidence to verify our theoretical insights.  

\textbf{Takeaway}: \textit{\ul{(1) MAE under different masking intensities learns representations of different abstraction levels; (2) Learning high-level representations is very hard with extreme masking.}}


\paragraph{Is current MAE optimal for representation learning?}
As reflected in the discussion above, although MAE offers the flexibility of tuning the masking scheme to learn representations of various levels, it is inherently challenging to learn high-level representations by random masking without prior knowledge of the latent structure. 
In contrast, contrastive learning~\cite{chen2020simple,he2020momentum,chen2021empirical,caron2021emerging,Caron2020UnsupervisedLO,bardes2021vicreg,zbontar2021barlow} actively leverages the prior knowledge encoded in data augmentations to extract the augmentation-invariant latent variables~\cite{von2021self} which correspond to the high-level latent variables in our hierarchical model.
Our theory suggests an explanation for why representations learned by contrastive learning are superior to those of MAE on high-level tasks like linear-probing classification. 

\textbf{Takeaway}: \textit{\ul{Learning high-level representations can be challenging for random masking.}}
\begin{figure*}[h]
\centering
\includegraphics[width=0.95\linewidth]{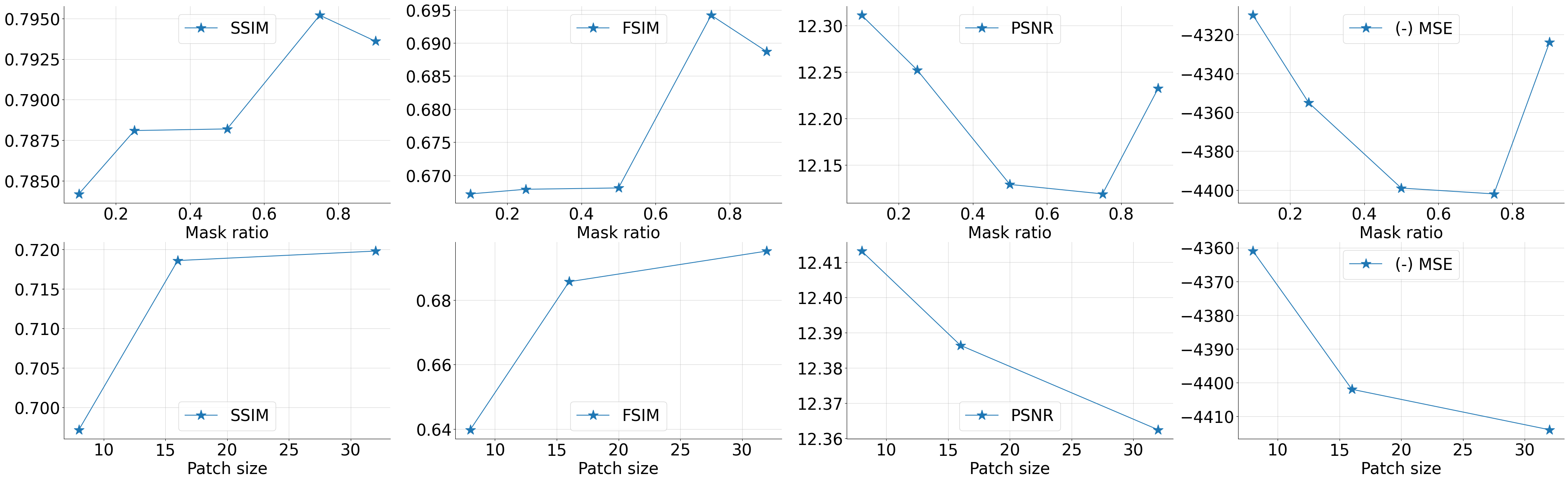}
\vspace{-.3em}
\caption{\textbf{Reconstruction evaluation} using the validation set without masking, based on two structural-level similarity metrics (SSIM and FSIM) and two pixel-level metrics (PSNR and MSE). We plot negative MSE for easier visualization. Higher SSIM and FSIM indicate high-level information is better captured, while higher PSNR and negative MSE indicates better low-level reconstruction.}
\vspace{-1em}
\label{fig:metrics}
\end{figure*}

%% file: sections/experiments.tex
\section{Experiments} \label{sec:experiments}
We conduct five sets of experiments and then provide insights into possible empirical improvements over MAE. 
We investigate the following question: \textit{how does the masking aggressiveness influence the representation?} 
To this end, we pretrain MAE using different masking ratios and making patch sizes, and then conduct the following evaluations: 1) measuring structure-level and pixel-level similarities between the reconstructed and the original images; 2) visualizing self-attentions to understand what is learned; 3) performing linear probing on ImageNet-1K (IN1K) and different ImageNet variants;
4) measuring the shape bias \cite{geirhos2018imagenet} which estimates how much a network leverages high-level shape information over low-level texture information; 
and 5) transfer learning on object detection and segmentation. 
Details of experiments can be found in Appendix.

\paragraph{Pretraining overview.} We conduct pretraining on IN1K using the MAE pipeline \cite{he2021masked}, with ViT-Base as the backbone of our study. We conduct two sets of pretraining: 1) fixing patch size at $16$ and varying the masking ratios from $\{0.1, 0.25, 0.5, 0.75, 0.9\}$. Larger masking ratios suggest larger portions of pixels being masked, i.e., $0.9$ suggests $90\%$ of pixels being randomly masked for the encoder. 2) Fix the masking ratio at $0.75$ and vary the patch size from $\{8, 16, 32\}$. 
To decouple the patch size for masking images and the patch size hyperparameter in the Vision Transformer, we adopt the implementation from \cite{hu2022exploring}. The patch size studied in this paper refers to the minimal \textit{masking unit} size, and the hyperparameter of the ViT patch size remains fixed at 8. 

\subsection{Reconstructing High-level or Low-level Representations}

\paragraph{Setup.} We begin our study by evaluating the high-level structural and low-level pixel-wise similarities between the reconstructed images from MAE and the original inputs. We choose two metrics for high-level similarities and two metrics for low-level similarities. If the structural similarities are high, MAE captures more perceivable structural semantics from the input. The two high-level similarities are structural similarity index measure \cite{wang2004image} (\textbf{SSIM}) and feature similarity index measure \cite{zhang2011fsim} (\textbf{FSIM}). Both metrics consider the change of perceptions in structural information \cite{johnson2016perceptual}. SSIM considers the normalized mean value of the structural similarity between the original and reconstructed images, and FSIM considers the normalized mean value of the feature similarity between the two images. A higher SSIM or a higher FSIM suggests a better reconstruction of high-level information (structural or feature-wise). On the other hand, if the pixel-level similarity between reconstructed images and the original input is high, then MAE is deemed to capture the low-level information about the input better. The two low-level metrics are the mean squared error (\textbf{MSE}), which is the squared differences between the original and reconstructed images in the pixel space, and the peak signal-to-noise ratio (\textbf{PSNR}), which measures the ratio between the power of the maximum possible pixel value and the power of corruption noise. A lower MSE or a \textit{higher} PSNR suggests a better reconstruction at the pixel level. Note that a very low MSE or a very high PSNR may also suggest that the model captures high-level information well. All four metrics are full reference, meaning the assessment is based on comparing original and reconstructed images rather than the reconstructed output. We introduce the high-level and low-level metrics below and perform the reconstructions on the IN1K evaluation set. The full details and comparisons of the four metrics can be found in \cite{sara2019image}.

\paragraph{Evaluation of image reconstructions.} We include the results in Figure \ref{fig:metrics}. 
We plot the negative of the MSE to show a consistent trend with PSNR, so higher means better low-level reconstruction. From the first row, varying masking ratios from $0.1$ to $0.75$, higher masking ratios produce reconstructions with higher structural information similarities with the original image (higher SSIM and FSIM), but the model trained with the extremely high ratio 0.9 captures more low-level information (higher PSNR and higher negative MSE). 
On the other hand, lower masking ratios tend to reconstruct images that capture low-level information better. 
From the second row, larger patch sizes produce image reconstructions that capture high-level similarities better, while smaller patch sizes have low-level metrics. 
The empirical observations validate our insight from Section \ref{para:masking_mae_theory}: \textit{higher masking ratios and patch sizes capture high-level structural information better}, but \textit{extreme masking ratios (both low and high) capture less high-level and more low-level information.} 

\subsection{Attention Analysis}
\begin{figure}[t]\centering
\includegraphics[width=0.96\linewidth]{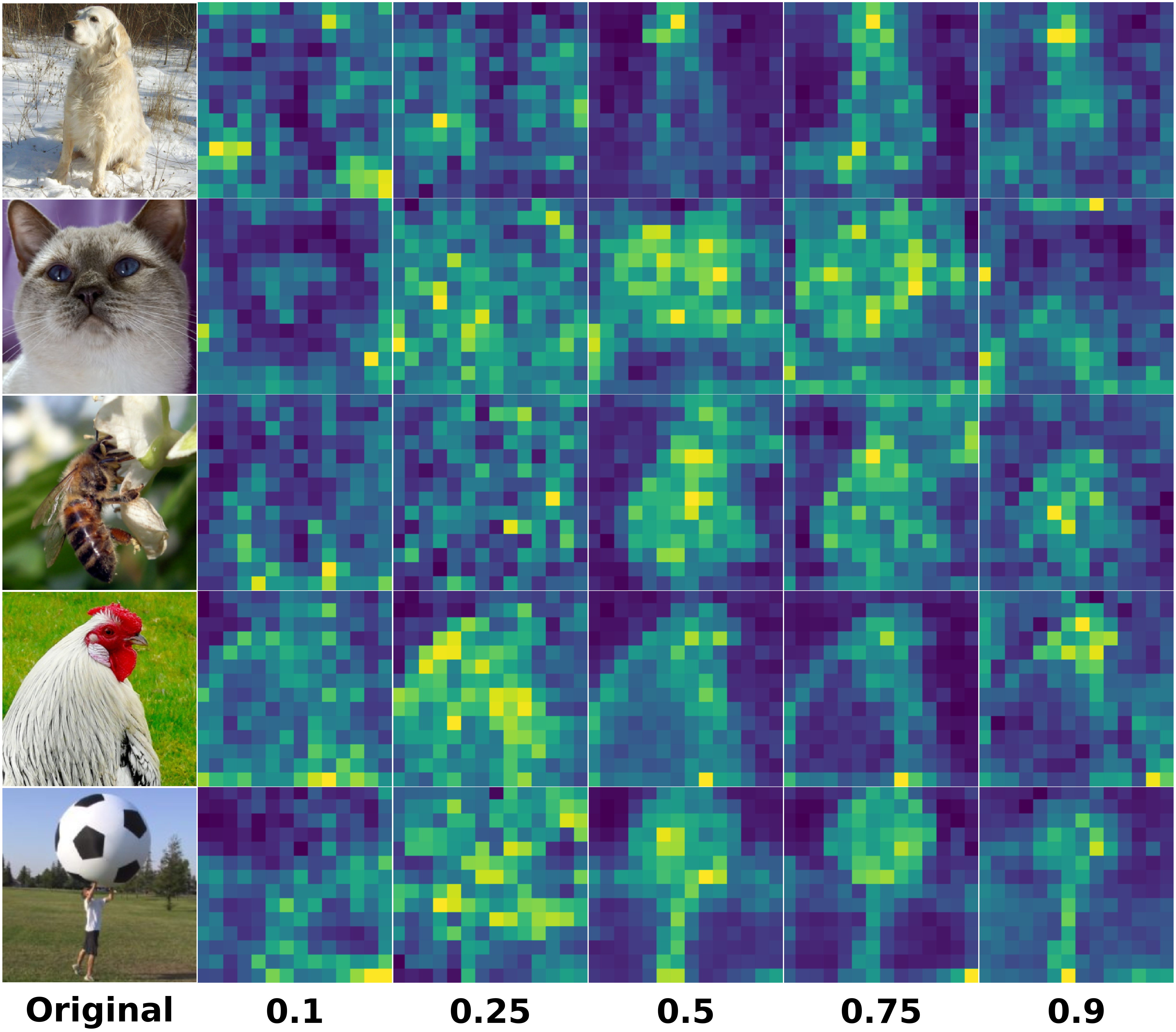}\vspace{-1em}
\caption{\textbf{Self-attention of the \texttt{[CLS]} tokens} averaged across the heads of the last layer in MAE.}
\label{fig:mask_dino}
\end{figure}

\begin{figure}[t]\centering
\includegraphics[width=0.96\linewidth]{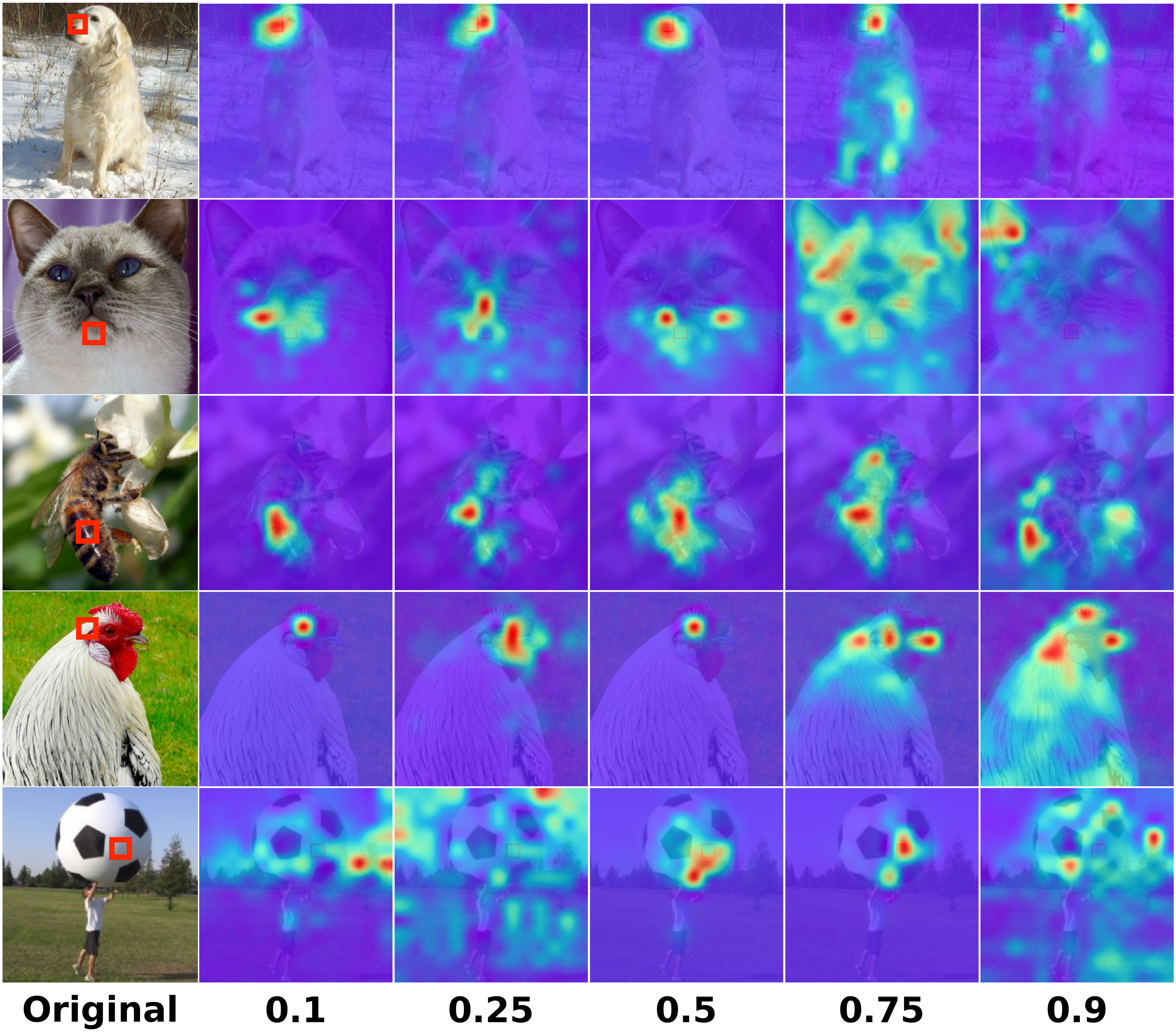}\vspace{-1em}
\caption{\textbf{Self-attention of an object-related token}. Chosen tokens are shown in red squares: dog nose, cat chin, bee abdomen, chicken head, and football center, respectively. }
\label{fig:mask_object_token}
\vspace{-1em}
\end{figure}

In this section, we measure the property of the learned representations of MAE by probing the attention heads. We would like to understand visually how masking ratios and patch sizes influence MAE's capacity to capture object-centric semantics.
We provide two types of visualization: self-attention on the \texttt{[CLS]} token and self-attention on an object-related token.  \texttt{[CLS]} has been considered a compact token to represent the whole image for downstream tasks, although recent work \cite{he2021masked} suggests that the average pooling of all tokens may achieve slightly better results. Therefore, we also provide an analysis of object-related tokens to evaluate if MAE can contextualize object information across tokens.


We plot examples of self-attention of the \texttt{[CLS]} token in Figure \ref{fig:mask_dino} and self-attention of non-CLS tokens related to the object in Figure \ref{fig:mask_object_token}.  From the visualizations, as the masking ratio increases from $10\%$ to $90\%$, the model is increasingly more able to grasp succinct information about the holistic objects rather than only focusing on the regions around the chosen token. However, extreme ratio $0.9$ contains more low-level information and background information and cannot capture most of the remaining tokens related to objects (e.g., the dog, cat, and bee images in Figure \ref{fig:mask_object_token}). Extremely low masking ratios such as $0.1$ capture both object-related and background tokens. Similarly, extreme masking ratios contextualize over other object-related tokens worse than intermediate masking ratios. 
We include the visualizations for patch sizes in Appendix. 
We observe that models trained with larger patch sizes better capture high-level information, but extreme patch size hurts, which validates our theoretical insight that \textit{moderate masking ratios and patch sizes are critical for MAE to learn succinct and comprehensive object information.}

\subsection{Representation Linear Separability}
\begin{figure*}[h]
\centering
\includegraphics[width=0.9\linewidth]{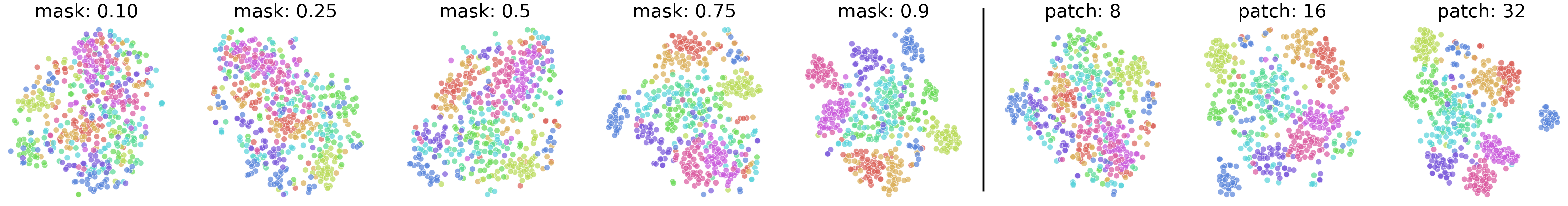}
\vspace{-1em}
\caption{T-SNE embeddings of different MAE models under varied masking ratios and patch sizes. We fix the patch size at $16$ to vary the masking ratios and fix the masking ratio at $0.75$ to change the patch sizes. Each color represents one ImageNet class. }
\label{fig:tsne}\vspace{-1em}
\end{figure*}
\paragraph{T-SNE embedding visualizations.} To gain a visual understanding of how masking ratios and patch sizes influence the representation structure, we visualize T-SNE~\cite{van2008visualizing} embeddings of different models. 
We randomly select ten classes from ImageNet. 
The results are shown in Figure~\ref{fig:tsne}. 
From $0.1$ to $0.75$, a larger masking ratio consistently produces a more linearly separable representation, while the linear separabilities of representations with masking ratios $0.75$ and $0.9$ are visually similar. 
For different patch sizes, the embeddings are more separated as the patch sizes grow. \textit{Non-extreme masking ratios and larger patch sizes generate more linearly separable embeddings.}

\paragraph{Linear probing on IN1K.} We use linear probing to test how linearly separable the features are in the learned MAE representation. We show the linear probing results in Table \ref{tab:linear_prob} in row 1N1K. For different masking ratios, similar to the observation in \cite{he2021masked}, the accuracy increases steadily until the masking ratio reaches the sweet point of 0.75. An extremely large masking ratio (0.9) hurts performance. For different patch sizes, which are not shown in \cite{he2021masked}, we observe that the accuracy increases first from 8 to 16, then decreases significantly when the patch size is 32. From the results, higher masking ratios and larger patch sizes perform better at linear probing than lower masking ratios, but extreme masking hurts linear probing.

\paragraph{Robustness evaluation on ImageNet variants.} We evaluate the robustness of the MAE models on different variants of ImageNet validation datasets, or object detection datasets that share similar class information with ImageNet-1K: ImageNet-v2 (INV2) \cite{shankar2020evaluating}, ObjectNet (OJN) \cite{barbu2019objectnet}, ImageNet-Adversarial (IN-A) \cite{hendrycks2021natural}, ImageNet-Rendition  \cite{barbu2019objectnet}, and ImageNet-Sketch (IN-S) \cite{wang2019learning}. 
These datasets share similar semantics and labels with ImageNet but are under different data distributions. 
The MAE models are first trained in a supervised fashion on IN1K for linear probing, and inference is run on the evaluation sets without any training. 
Table~\ref{tab:linear_prob} shows for all evaluation datasets, a reasonably large masking ratio (i.e., $0.75$) achieves better robustness than smaller (i.e., $0.25$) masking ratios, although extremely large ($0.9$) or small ($0.1$) masking ratios hurt the performance. For patch sizes, larger patch sizes yield better robustness evaluations on IN-v2, OJN, IN-R, and IN-S. 
\textit{Non-extreme masking ratios and large patch sizes have stronger robustness performances than extreme masking ratios or patch sizes.}

\begin{table}
\vspace{-.5em}
\tablestyle{3pt}{1.1}
\begin{tabular}{llllllll}
\multirow{1}{*}{mask ratio} &
\multirow{1}{*}{patch size}
& \multicolumn{1}{c}{IN1K} & \multicolumn{1}{c}{IN-v2} & \multicolumn{1}{c}{OJN} & \multicolumn{1}{c}{IN-R} & \multicolumn{1}{c}{IN-A}  & \multicolumn{1}{c}{IN-S} \\
\shline
0.1 & 16 & 47.45 & 34.72 & 9.42 & 14.63 & 2.00 & 7.25 \\
0.25 & 16 & 53.58 & 40.34 & 11.54 & 18.68 & 2.49 & 10.27 \\
0.5 & 16 & 60.07 & 46.71 & 13.94 & 22.44 & 2.89 & 12.58 \\
0.75 & 16 & 67.41 & 54.23 & 18.24 & 25.20 & 3.76 & 15.51 \\
0.9 & 16 & 62.97 & 49.52 & 15.87 & 19.11 & 2.76 & 10.46 \\
\shline
0.75 & 8  & 62.57 & 49.17 & 13.44 & 19.42 & 3.73 &  10.73 \\
0.75 & 16 & 68.96 & 55.94 & 13.73 & 24.23 & 6.29 & 18.81 \\
0.75 & 32 & 73.31 & 61.35 & 19.03 & 27.84 & 12.69 & 28.30 \\

\end{tabular}
\vspace{-1em}
\caption{\textbf{Accuracy ($\%$) of linear probing and robustness evaluation} on ImageNet variants and ObjectNet. We linear-probe MAE via supervised training on IN1K, and then perform inference on IN1K as well as other evaluation sets.}
\vspace{-2em}
\label{tab:linear_prob}
\end{table}

\subsection{Shape Bias}
\vspace{-.5em}
\paragraph{Texture vs. shape bias.} Next, we analyze to what extent different MAE models rely on high-level vs. low-level information. We follow the analysis in \cite{geirhos2018imagenet}, where the authors study whether a model leverages more low-level textures than high-level shapes for classification. As shown in Table~\ref{tab:shape_bias}, intermediate masking ratios (i.e., $0.25, 0.5$, and $0.75$) show a high level of shape bias, suggesting that the corresponding models exploit more high-level shape information.
In contrast, extreme masking ratios (i.e., $0.1$ and $0.9$) leverage more low-level textures. 
This suggests that \textit{extreme masking schemes make it more difficult to capture high-level shapes for MAE.}
\begin{table}[t]
\tablestyle{5pt}{1.05}
\begin{tabular}{cccccccccc}
mask ratio & 0.1 & 0.25 & 0.5 & 0.75 & 0.9 \\
\shline
shape bias & 0.1352 & 0.2545 & 0.2458 & 0.2563 & 0.2014\\

\end{tabular}
\caption{\textbf{Shape bias} \cite{geirhos2018imagenet} measurement, a higher metric indicates that the model classifies images relying on the high-level shape feature rather than the low-level texture feature.}
\label{tab:shape_bias} \vspace{-.7em}
\end{table}



\subsection{Transfer Learning}
Next, we evaluate the quality of MAE models on different downstream tasks. Specifically, we look at object detection and segmentation on the COCO dataset \cite{lin2014microsoft}, which requires a strong semantic understanding of the scenes. 
We finetune Mask R-CNN \cite{he2017mask} end-to-end using MAE-pretrained ViT weights. Following the practice in \cite{he2021masked}, we adapt the ViT backbone to make it compatible with FPN \cite{lin2017feature}. 
In Table \ref{tab:coco}, we report box AP for object detection and mask AP for instance segmentation. 
We reduce the number of epochs to $45$ due to computational constraints. 
We observe that the $0.75$ masking ratio yields the best detection and segmentation average precision, suggesting that the masking ratio $0.75$ generates representation with the best semantic understanding. 
The extremely high masking ratio of $0.9$ and a low masking ratio of $0.1$ hurt the performance. Results of different patch size experiments are included in Appendix. 
The results suggest that \textit{higher, but not extreme, masking ratios generate the best representation of object detection and segmentation tasks.}
\begin{table}[t]
\vspace{-.7em}
\tablestyle{5pt}{1.05}
\begin{tabular}{llcc}
mask ratio & mask size & AP$^\text{box}$ & {AP$^\text{mask}$}\\ 
\shline
0.1 &  16 & 30.47 & 28.24 \\
0.25 &  16 & 32.38 & 29.95 \\
0.5 &  16 & 34.87 & 32.11 \\
0.75 &  16 & \bf 39.72 & \bf 36.35\\
0.9 &  16 & 37.17 & 34.35\\

\end{tabular}
\vspace{-.7em}
\caption{\textbf{COCO object detection and segmentation} using a ViT Mask R-CNN baseline. }
\vspace{-1.5em}
\label{tab:coco} 
\end{table}






\subsection{Potential Algorithmic Improvements}
Lastly, we discuss empirical suggestions based on our results that could benefit the performance of MAE. 

First, as discussed in Section~\ref{sec:theory}, when reconstructing the masked pixels near the boundary between the masked and unmasked regions, the model uses nearby visible pixels to interpolate, therefore capturing low-level pixel information. If high-level representation is desired for downstream tasks, the boundary pixels may be ignored when calculating the objective function.

Next, in light of the limitation of random masking in Section~\ref{sec:theory}, one may leverage the latent structure of the underlying data-generating process for masking designs, which can serve as a more principled approach than recent work that exploits auxiliary information for masking~\cite{shi2022adversarial,li2022semmae, kakogeorgiou2022hide,li2021mst}. To this end, one may take advantage of the recent development of causal discovery \cite{huang2022latent,xie2022identification} to identify the latent structure. 

Lastly, if low-level information is preferable for downstream tasks, an extremely high masking ratio can retain such information and is more computationally efficient than its low masking ratio counterpart. 

%% file: sections/related_work.tex
\section{Related work}

\subsection{Masked Autoencoders}
Masked image modeling (MIM)~\cite{baevski2022data2vec,anandkumar2013learning,chen2020generative,xie2022simmim,he2021masked,zhou2021ibot,bao2021beit,chen2022context,li2021mst} has been gaining momentum recently due to their sota-of-the-art performances over many downstream tasks. 
The pretraining objective is simple in its basic form: the model is tasked to predict the masked-out image pixels with the information of the unmasked part.
Despite the simplicity of the task, many intriguing properties have been observed on MIM that escape rigorous analysis.
For instance, small masking ratios and masking patch sizes are empirically shown detrimental to downstream tasks like classification~\cite{he2021masked,hu2022exploring}.
It is hypothesized that aggressive masking forces to model to leverage more global information, rather than local interpolation~\cite{he2021masked}.
However, whether such intuition is theoretically justifiable remains elusive. 
In this work, we provide theoretical verification of such intuitions and further derive insights into MAE's empirical behavior.

\subsection{Theoretical Understanding of MAE}
Despite the prominent success of MAE, only a limited number of papers are dedicated to understanding its underlying mechanism in a principled manner~\cite{lee2021predicting,cao2022understand,pan2022towards,zhang2022mask}.
Lee et al.~\cite{lee2021predicting} establish the connection between the inpainting pretraining task and downstream tasks by assuming that the downstream task target captures the statistical dependency between the visible part and the masked part in the inpainting.
Under this assumption, they show that the sampling complexity of the downstream task can be largely reduced by pretraining.
Cao et al.~\cite{cao2022understand} inquire into the interactions between the transformer architecture and the MAE representation, highlighting the critical role of the attention mechanism in the success of MAE.
Pan et al.~\cite{pan2022towards} make a multi-view assumption on the samples, showing that MAE can extract class-relevant semantics with shallow convolutional models.
Zhang et al.~\cite{zhang2022mask} study masking through the data-augmentation perspective and employ the augmentation graph~\cite{haochen2021provable} to illustrate the impact of masking on downstream task performance.
In contrast, our work employs the hierarchical latent variable model, which lets us directly examine the relationship between the masking operation and the learned representations. 
Also, our theoretical guarantee is on the statistical identifiability of the true data-generating process rather than the statistical/optimization complexities as in most prior work.

\subsection{Identifiability Guarantees for Nonlinear Latent-variable Models}
In unsupervised learning, identifiability means latent variables involved in the underlying data-generating process can be estimated from observational data. This is critical to tasks like feature disentanglement~\cite{higgins2017betavae,Karras_2019_CVPR,burgess2018understanding,huang2018multimodal,xie2022identification} in the image generation community.
However, principled disentanglement in the non-linear regime is challenging and even proved impossible without additional assumptions on the data-generating process~\cite{locatello2019challenging}.
Recent advances in independent component analysis (ICA)~\cite{comon1994independent,bell1995information,ICAbook} obtain identifiability in the non-linear regime by imposing additional constraints on either the latent variable distribution or the function class variables~\cite{hyvarinen2019nonlinear, sorrenson2020disentanglement, halva2020hidden, lachapelle2021disentanglement, khemakhem2020variational, von2021self, kong2022partial, zheng2022identifiability, lyu2021latent}.
Most relevant to ours are the identifiability theories in \cite{von2021self,lyu2021latent} in which similar latent causal models (Figure~\ref{fig:cs_model}) are studied.
Specifically, our model allows the generating functions $g_{\mm} \neq g_{\mm^{c}} $ to be distinct (cf. identical functions assumed in \cite{von2021self}) and statistical dependence between $ \cc $ and $ \sss_{\mm^{c}} $ (cf. independence assumed in \cite{lyu2022understanding}).
Additionally, both works~\cite{von2021self,lyu2022understanding} focus on contrastive learning with data augmentation, while our subject is MAE.

%% file: sections/conclusion.tex
\section{Conclusion}

In this work, we formulate the data-generating process as a hierarchical latent variable model and provide guarantees that MAE can identify the true variables in such a hierarchical latent model.
We then show how different masking ratios and patch sizes determine the set of true latent variables to be recovered, which influences the representation abstractions learned in MAE. Empirically, we show that non-extreme masking ratios or patch sizes often capture succinct and robust high-level information, while extreme masking ratios capture more low-level information. 




%% file: sections/acknowledgement.tex
\vspace{-.5em}
\paragraph{Acknowledgement} \small{We thank the Google TRC program for the TPU Research Cloud support, Ronghang Hu and Xinlei Chen for the MAE TPU code, Biwei Huang for technical discussions, Tao Lin for feedback on the manuscript, and anonymous reviewers for valuable feedback. 
The work of LK and YC is supported in part by NSF under the grants CCF-1901199 and DMS-2134080, and by ONR under the grant N00014-19-1-2404.
The work of MM and LP is partially supported by BMW, National Science Foundation awards 1722822 and 1750439, and National Institutes of Health awards R01MH125740, R01MH096951, R21MH130767 and R01MH132225. 
This project is also partially supported by the National Institutes of Health (NIH) under Contract
R01HL159805, by the NSF-Convergence Accelerator Track-D award 2134901, by a grant from
Apple Inc., a grant from KDDI Research Inc, and generous gifts from Salesforce Inc., Microsoft
Research, and Amazon Research. }


%% file: sections/appendix.tex
\section{Proof for Theorem~\ref{thm:locating_c}} \label{sec:proof_theorem_graph}
In this section, we provide the proof for Theorem~\ref{thm:locating_c}. 

\locatec*

\begin{algorithm}[H]
	\begin{algorithmic}[1]
        \myState{\textbf{inputs}: The hierarchical graph structure $\mG$, and the partitioned observables $ \cX_{\mm} $, $ \cX_{\mm^{c}} $. }
        \myState{ $ \cC, \cS_{\mm} \leftarrow \emptyset, \emptyset $. }
        \myState{\textbf{Selection stage:}}
        \For{ $\xx \in \cX_{\mm}$}
            \myState{ $ \cZ \leftarrow \{ \xx \} $. }
            \While{ $ \cZ \neq \emptyset $ }
                \myState{ $\cZ, \cE \leftarrow \text{LocateParents} ( \cZ ) $}
                \myState{ $ \cS_{\mm} \leftarrow \cS_{\mm} \cup \cE $ }
                \For{ $\pp \in \cZ$ }
                    \If{ $ \pp \in \text{Ancestors} ( \xx_{\mm^{c}} ) $ }
                        \myState{ $ \cC \leftarrow \cC \cup \{ \pp \} $ }
                        \myState{ $ \cZ \leftarrow \cZ \setminus \{ \pp \} $ }
                    \EndIf
                \EndFor
            \EndWhile
        \EndFor
        \myState{\textbf{Pruning stage:}}
        \For{ $ \dd \in \cC $ }
            \For{ $ \dd' \in \cC \setminus \{ \dd \}$ }
                \If{ $ \dd' \in \text{DirectedPaths} (\dd, \cX_{\mm^{c}}) $ }
                    \myState{$ \cC \leftarrow \cC \setminus \{ \dd \} $}
                \EndIf
            \EndFor
        \EndFor
    \Return{ $\cC$, $ \cS_{\mm} $ }
	\end{algorithmic}

	\mycaptionof{algorithm}{
        \footnotesize
        \textbf{Search for the minimal $\cc$ and $\sss_{\mm}$.}
        $\cc$ and $\sss_{\mm}$ discussed in text can be viewed as the concatenations of vectors in $\cC$ and $\cS_{\mm}$.
        $\text{LocateParents}(\cdot)$ pins down the locations $\cZ$'s parents (including exogenous variables) in the graph.
        $\text{DirectedPaths} (\dd, \cX_{\mm}^{c}) $ returns the set of variables on the directed paths between $\dd$ and $ \cX_{\mm}^{c} $.
        }
	\label{alg:locate_c}
\end{algorithm}

\begin{algorithm}[H]
	\begin{algorithmic}[1]
        \myState{\textbf{inputs}: The hierarchical graph structure $ \mG $, the partitioned observables $ \cX_{\mm} $, $ \cX_{\mm^{c}} $, and $\cC$ returned by Algorithm~\ref{alg:locate_c}. }
        \myState{ $ \cS_{\mm^{c}} \leftarrow \emptyset $. }
        \For{ $\xx \in \cX_{\mm^{c}}$}
            \myState{ $ \cP, \cP' \leftarrow \{ \xx \}, \emptyset $. }
            \While{ $ \cP \neq \emptyset $ }
                \For{ $\pp \in \cP$ }
                    \For{ $ \pp' \in \text{LocateParents} (\pp)$ }
                        \If{ $ \pp' $ is exogenous }
                            \myState{ $ \cS_{\mm^{c}} \leftarrow \cS_{\mm^{c}} \cup \{ \pp' \} $ }
                        \ElsIf{ $ \pp' \in \cC $ }
                            \myState{
                                $ \cS_{\mm^{c}} \leftarrow \cS_{\mm^{c}} \cup ( \text{LocateParents} (\pp) \setminus \{ \pp' \} ) $
                            }
                        \Else
                            \myState{ 
                                $ \cP' \leftarrow \cP' \cup \{ \pp'\} $
                             }
                        \EndIf
                    \EndFor
                \EndFor
                \myState{ $\cP \leftarrow \cP'$ }
            \EndWhile
        \EndFor
        \Return{ $ \cS_{\mm^{c}} $ }
	\end{algorithmic}

	\mycaptionof{algorithm}{
	    \textbf{Search for $\sss_{\mm^{c}}$ given $\cC$.}
        \small
        $\text{LocateParents}(\pp)$ pins down the locations $\pp$'s parents (including exogenous variables) in the graph.
        }
	\label{alg:locate_smc}
\end{algorithm}


\begin{proof}
    We will show that Algorithm~\ref{alg:locate_c} returns the minimal set of variables that satisfy all conditions in Theorem~\ref{thm:locating_c}, which implies its existence.
    We will then argue that such $\cC$ is unique for a specific mask $\mm$.


    \paragraph{Condition 1:}
    We first discuss the invertibility of $g_{\xx_{\mm}}$. 
    Due to the invertibility assumption of the generating process, each backtrack step in Algorithm~\ref{alg:locate_c} is invertible (lossless).
    Thus, before the pruning stage, the mapping between $(\cC, \cS_{\mm})$ and $ \cX_{\mm} $ is invertible, as the information of $\cX_{\mm}$ is either stored in either $\cC$ or $ \cS_{\mm} $.
    We now show that the pruning stage does not break this invertibility.
    To see this, we note that for each $ \cc $ that is removed in the pruning stage, there exists $ \cc' \in \cC $ on the directed path from $ \cc $ to $ \cX_{\mm^{c}} $ (per Algorithm~\ref{alg:locate_c}).
    Therefore, $ \cc $ is a parent/ancestor of $ \cc' $ and can thus be retrieved by backtracking from $ \cc' $ thanks to the invertibility of the generating process.
    Therefore, the mapping between $(\cC, \cS_{\mm})$ and $ \cX_{\mm} $ is invertible.

    We now address the invertibility of $g_{\xx_{\mm^{c}}}$, i.e., the mapping between $(\cC, \cS_{\mm^{c}})$ and $ \cX_{\mm^{c}} $.
    We observe that a similar argument applies: Algorithm~\ref{alg:locate_smc} dictates that the latent variables from the backtracking from $ \cX_{\mm^{c}} $ are either stored in either $ \cC $ or $ \cS_{\mm^{c}} $. It follows that $ g_{\xx_{\mm^{c}}} $ is invertible.

    \paragraph{Condition 2:}
    We show that $(\cc, \sss_{\mm}, \sss_{\mm^{c}})$ returned by Algorithm~\ref{alg:locate_c} and Algorithm~\ref{alg:locate_smc} satisfies Condition 2 by contradiction.
    We suppose that $ \sss_{\mm} \not\ind (\cc, \sss_{\mm^{c}})$.
    Then it implied that $ \exists \dd \in ( \cc, \sss_{\mm^{c}} )$, $ \exists \bm\varepsilon \in \sss_{\mm} $, such that $ \dd \in \text{Descendants} (\bm\varepsilon) $.
    More precisely, it followed that there was a directed path that started from $ \bm\varepsilon $ and ended at $ \dd $, and a child of $ \bm\varepsilon $, denoted as $ \bm\delta $, was located on this path.
    If $\dd \not\in \text{Descendants} (\bm\varepsilon) $, there would be no directed paths from $ \bm\varepsilon $ to $\dd$ and thus at least one V-structure would sit on each path between $ \bm\varepsilon $ and $\dd$ that blocked the path. 
    According to Algorithm~\ref{alg:locate_c}, as $ \bm\varepsilon \in \sss_{\mm} $, it implied that $ \bm\delta \not\in \cc$ and $ \bm\delta \not\in \text{Ancestors} (\xx_{\mm}) \cap \text{Ancestors} (\xx_{\mm^{c}}) $.

    We first investigate the case where $ \dd \in \cc $, i.e., $ \sss_{\mm} \not\ind \cc $.
    The fact that $\dd \in \cc$ implied that $ \dd \in \text{Ancestors} (\xx_{\mm}) \cap \text{Ancestors} (\xx_{\mm^{c}}) $ which further implied that $ \bm\delta \in \text{Ancestors} (\xx_{\mm}) \cap \text{Ancestors} (\xx_{\mm^{c}}) $ as $ \bm\delta $ was an ancestor of $ \dd $.
    Therefore, we have arrived at a contraction to the observation that $ \bm\delta \in\text{Ancestors} (\xx_{\mm}) \cap \text{Ancestors} (\xx_{\mm^{c}}) $.

    We now discuss the scenario where $ \dd \in \sss_{\mm^{c}} $.
    By design, Algorithm~\ref{alg:locate_smc} ensures that $ \sss_{\mm^{c}} $ contains two types of latent variables, exogenous variables and a spouse of latent variables in $ \cc $.
    As $\sss_{\mm}$ consists solely of exogenous variables and exogenous variables are independent mutually, it could only be the case that $ \dd $ was a spouse of a latent variable in $ \cc $.
    By Algorithm~\ref{alg:locate_c}, there would be a directed path from $\bm\delta$ to $ \xx_{\mm} $.
    Also, Algorithm~\ref{alg:locate_smc} ensured that $ \dd $ lied on a path directed to $ \xx_{\mm^{c}} $.  
    As there existed a directed path from $ \bm\delta $ to $ \dd $, there must exist a directed path from $ \bm\delta $ to $ \xx_{\mm^{c}} $.
    Therefore, $ \bm\delta \in \text{Ancestors} (\xx_{\mm}) \cap \text{Ancestors} (\xx_{\mm^{c}}) $ which contradicts the fact established above.
    
    Therefore, these contradiction implies that $ \sss_{\mm} \ind (\cc, \sss_{\mm^{c}}) $.

    So far, we have shown that Algorithm~\ref{alg:locate_c} and Algorithm~\ref{alg:locate_smc} yield $ (\cc, \sss_{\mm}, \sss_{\mm^{c}}) $ that fulfills the conditions of Figure~\ref{fig:cs_model}.
    In the following, we show that $ (\cc, \sss_{\mm}) $ is the minimal solution and is unique.




    \paragraph{Uniqueness and minimality of $(\cc, \sss_{\mm})$:}
    We now reason about that given the mask and the hierarchical structure, $ (\cc, \sss_{\mm}) $ returned by Algorithm~\ref{alg:locate_c} is the set of minimal dimensionality that can fulfill the conditions, and such a minimal set is unique.

    By construction, Algorithm~\ref{alg:locate_c} ensures that for each $ \cc \in \cC $ there exists an undirected path that is made up of a directed path from $\cc$ to the masked variable $ \xx_{\mm} $ and a directed path from $\cc$ to the unmasked variable $ \xx_{\mm^{c}} $ and no other $ \cc' \in \cC $ sits on this entire undirected path.
    To see this, there must exist a directed path from $\cc$ to $ \xx_{\mm} $ without any other $ \cc' \in \cC$ on it, otherwise $ \cc $ would not be placed in $ \cC $ in Algorithm~\ref{alg:locate_c}.
    In addition, the pruning stage of Algorithm~\ref{alg:locate_c} mandates that there must exist $ \xx_{\mm^{c}} $ such that the path from $ \cc $ to $ \xx_{\mm^{c}} $ does not contain other $ \cc' \in \cC $.
    We note that $ \cc $ chosen by Algorithm~\ref{alg:locate_c} is the variable with the smallest possible dimension to block such a path, as it resides on the highest level compared to other variables on the path and the variable dimension increases monotonically along directed paths.

    Therefore, the choice of each $ \cc $ is minimal, and such a choice is unique.
    As $ \cS_{\mm} $ is the set of exogenous variables necessary for $\cC$ to restore $ \cX_{\mm} $, the selection of $ \cS_{\mm} $ is also unique.
    Hence, we conclude that the $ (\cC, \cS_{\mm}) $ returned by Algorithm~\ref{alg:locate_c} is the minimal choice and is unique.
    





\end{proof}

\section{Identifiability proof}
\label{sec:identifiability_proof}
In this section, we present the proof for Theorem~\ref{thm:mae_identifiability}.
We first give a general identifiability theory (i.e., Theorem~\ref{thm:generative_identifiability}) for the generating process in Figure~\ref{fig:cs_model} and then make the connection to the proof of Theorem~\ref{thm:mae_identifiability}.

\begin{theorem} \label{thm:generative_identifiability}
    The generating process in Figure~\ref{fig:cs_model} is defined as follows:
    \begin{align} \label{eq:general_invertible_process}
        [\vv_{1}, \vv_{2}] &= g( \cc, \sss_{1}, \sss_{2} ), \\
        \vv_{1} & = g_{1} (\cc, \sss_{1}), \\
        \vv_{2} & = g_{2} (\cc, \sss_{2}),
    \end{align}
    where $ \cc \in \cC \subset \R^{d_{c}} $, $ \sss_{1} \in \cS \subset \R^{d_{s_{1}}} $, and $ \sss_{2} \in \cS_{2} \subset \R^{d_{s_{2}}}$. 
    Both $g_{1}$ and $g_{2}$ are smooth and have non-singular Jacobian matrices almost anywhere, and $g$ is invertible.
    
    If $ \hat{g}_{1}: \cZ \to \cV_{1} $ and $ \hat{g}_{2}: \cZ \to \cV_{2} $ assume the generating process of the true model $(g_{1}, g_{2})$ and match the joint distribution $p_{\vv_{1}, \vv_{2}}$, 
    then there is a one-to-one mapping between the estimate $\hat{\cc}$ and the ground truth $\cc$ over $ \cC \times \cS \times \cS $, that is, $\cc$ is block-identifiable.
\end{theorem}

\begin{proof}
    For $(\vv_{1}, \vv_{2}) \sim p_{\vv_{1}, \vv_{2}} $, because of the matched joint distribution, we have the following relations between the true variables $ (\cc, \sss_{1}, \sss_{2}) $ and the estimated ones $ (\hat{\cc}, \hat{\sss}_{1}, \hat{\sss}_{2}) $:
    \begin{align}
        \vv_{1} &= g_{1} ( \cc, \sss_{1} ) = \hat{g}_{1} ( \hat{\cc}, \hat{\sss}_{1} ), \label{eq:v1_generation_1}\\
        \vv_{2} &= g_{2} ( \cc, \sss_{2} ) = \hat{g}_{2} ( \hat{\cc}, \hat{\sss}_{2} ) , \label{eq:v2_generation_1} \\
        (\hat{\cc}, \hat{\sss}_{1}, \hat{\sss}_{2}) &= \hat{g}^{-1} ( \vv_{1}, \vv_{2} ) = \hat{g}^{-1} ( g (\cc, \sss_{1}, \sss_{2}) ) := h(\cc, \sss_{1}, \sss_{2}), \label{eq:h_definition_1} 
    \end{align}
    where we define the smooth and invertible function $ h:= \hat{g}^{-1} \circ g $ that transforms the true variables $ (\cc, \sss_{1}, \sss_{2}) $ to estimates $ (\hat{\cc}, \hat{\sss}_{1}, \hat{\sss}_{2}) $.

    Plugging Equation~\ref{eq:h_definition_1} into Equation~\ref{eq:v1_generation_1} yields the following:
    \begin{align*}
        g_{1} ( \cc, \sss_{1} ) = \hat{g}_{1} ( h_{c, s_{1}} (\cc, \sss_{1}, \sss_{2}) ).
    \end{align*}
    For $ i \in \{1, \dots, d_{v_{1}} \} $ and ($ j \in \{ 1,\dots, d_{s_{2}}\}$), taking partial derivative of the $ i $-th dimension of both sides w.r.t. $ s_{2, j} $:
    \begin{align*} 
        0 = \frac{ \partial g_{1, i} ( \cc, \sss_{1} )}{ \partial s_{2, j} } = \frac{ \partial \hat{g}_{1, i} ( h_{c, s_{1}} ( \cc, \sss_{1}, \sss_{2}) ) }{ \partial  s_{2, j} }.
    \end{align*}
    The equation equals zero because there is no $ s_{2, j} $ in the left-hand side of the equation.
    Expanding the derivative on the right-hand side gives:
    \begin{align}
        \sum_{k \in \{ 1, \dots, d_{c} + d_{s_{1}} \} } \frac{ \partial \hat{g}_{1, i} }{ \partial h_{(c, s_{1}), k} } \cdot \frac{ \partial h_{(c, s1), k} }{ \partial {s}_{2, j} } ( {\cc}, {\sss}_{1}, {\sss}_{2} ) = 0
        \label{eq:one_equation_1}
    \end{align}
    For $(\hat{\cc}, \hat{\sss}_{1}) \in \cC \times \cS \setminus \cE_{1}$ where $\cE_{1}$ denotes some subset with zero measure, there are at least $ d_{c} + d_{s_{1}} $ values of $i$ for which vectors $ [ \frac{ \partial \hat{g}_{1, i} }{ \partial h_{(c, s_{1}), 1} } (\hat{\cc}, \hat{\sss}_{1}), \dots, \frac{ \partial \hat{g}_{1, i} }{ \partial h_{(c, s_{1}), d_{c}+d_{s_{1}}} } (\hat{\cc}, \hat{\sss}_{1}) ]$ are linearly independent, which is equivalent to the non-singular Jacobian matrix condition.
    Therefore, the $ (d_{c} + d_{s_{1}}) \times (d_{c} + d_{s_{1}}) $ linear system is invertible and the solution states that:
    \begin{align*}
        \frac{ \partial h_{(c, s_{1}), k} }{ \partial s_{2, j} } ( \cc, \sss_{1}, \sss_{2} ) = 0,
    \end{align*}
    for any $k \in \{1, \dots, d_{c} + d_{s_{1}} \}$, $j \in \{1, \dots, d_{s_{2}} \}$, and $ (\hat{\cc}, \hat{\sss}_{1}) \in \cC \times \cS \setminus \cE_{1}$.
    Therefore, we have shown that $ h_{c, s_{1}} $, i.e. $(\hat{\cc}, \hat{\sss}_{1})$, does not depend on $ {\sss}_{2} $.
    
    Applying the same reasoning to $ h_{c, s_{2}} $, we can obtain that $ h_{c, s_{2}} $, i.e. $(\hat{\cc}, \hat{\sss}_{2})$ does not depend on $ {\sss}_{1} $ on $\cC \times \cS $.
    
    Thus, for $ (\hat{\cc}, \hat{\sss}_{1}, \hat{\sss}_{2}) \in \cC \times \cS \times \cS $, we can observe that $\hat{\cc} $ does not depend on $ \sss_{1} $ and $ \sss_{2} $, that is, $ \hat{\cc} = h_{c} (\cc) $.

    Notice that in all procedures above, the roles of the true quantities $ (\cc, \sss_{1}, \sss_{2}, g, g_{1}, g_{2}) $ and the estimated quantities $ (\hat{\cc}, \hat{\sss}_{1}, \hat{\sss}_{2}, \hat{g}, \hat{g}_{1}, \hat{g}_{2}) $ are symmetric. 
    Therefore, we can switch the two sets of quantities and derive the relation: for $ (\cc, \sss_{1}, \sss_{2}) \in (\cC \times \cS \times \cS) $, $ \cc $ does not depend on $ \hat{\sss}_{1} $ and $ \hat{\sss}_{2} $, that is, $ \cc = h'_{c} (\hat{\cc}) $.

    
    In sum, we have shown that on $ (\cC \times \cS \times \cS) $, there is a one-to-one mapping between $ \cc $ and $ \hat{\cc} $.

\end{proof}

We now show that Theorem~\ref{thm:mae_identifiability} follows directly from Theorem~\ref{thm:generative_identifiability}.

\globalidentifiability*

\begin{proof}
    We invoke Theorem~\ref{thm:generative_identifiability} and establish the connection between the MAE training and the estimation model in Theorem~\ref{thm:generative_identifiability}.
    In particular, we show that under Assumption~\ref{assumption:mae_model}, any solution produced by the MAE objective satisfies the conditions in Theorem~\ref{thm:generative_identifiability} and consequently is equipped with the identifiability guarantee.

    We establish the correspondence between the MAE configuration and the estimation models in Theorem~\ref{thm:generative_identifiability}:
    \begin{itemize}
        \item $ \vv_{1} \leftarrow \xx_{\mm}$;
        \item $ \vv_{2} \leftarrow \xx_{\mm^{c}} $ ;
        \item $ \hat{g}_{1} \leftarrow D_{\mm} (\cdot, \hat{\sss}_{\mm}) $;
        \item $ \hat{g}_{2} \leftarrow \tilde{g}_{\mm^{c}} $, where $ E_{\mm^{c}}(\cdot) = [\tilde{g}^{-1}_{\mm^{c}} (\cdot)]_{1:d_{c}} $.
    \end{itemize}

    We can observe that the minimizer of MAE satisfies the conditions specified in Theorem~\ref{thm:generative_identifiability}.
    This is because for the optimal solution $E_{\mm^{c}}$ of the MAE objective, we can always construct a $\tilde{g}_{\mm^{c}}$, which, together with $D_{\mm}$, matches the joint distribution $p_{\xx_{m}, \xx_{\mm^{c}}}$ and shares $\hat{\cc}$, as stipulated in Theorem~\ref{thm:generative_identifiability}.
    Thus, as shown in Theorem~\ref{thm:generative_identifiability}, there exists a one-to-one mapping between the MAE estimate $ \hat{\cc} := E_{\mm^{c}} (\xx_{\mm^{c}}) $ and the true variable $\cc$, which concludes our proof. 
\end{proof}

\section{Experimental Setup}
In this section, we provide the details of the experimental setups for our empirical results. Checkpoints and some codes are in \url{https://github.com/martinmamql/mae_understand}.

\subsection{Masked Autoencoder}
Masked Autoencoder (MAE) is an auto-encoding approach based on Vision Transformers (ViT) \cite{dosovitskiy2020image}. It consists of five steps: masking, encoding, unmasking, decoding, and reconstruction. First, an image is divided into non-overlapping patches. Then MAE samples a subset of patches and discards the remaining patches. MAE uses a hyper-parameter, masking ratio, to determine the percentage of patches to discard. For instance, if the masking ratio is $75\%$, $\frac{3}{4}$ of the patches in an image will be discarded, and only $\frac{1}{4}$ of the patches will be fed into the encoder. The sampling of patches follows a uniform distribution. Next, a ViT encoder first embeds patches using a linear projection with positional embeddings and then uses the processed embeddings to feed into transformer blocks. For decoding, MAE first re-arranges the encoded embeddings from the visible patches according to their corresponding positions in the original image and then uses a shared learned mask token to fill in the patches that are masked. Essentially, this means the input of the decoder is a combination of encoded visible patches and the mask tokens, where the positions of the mask tokens are the masked patches in the original image. The decoder is another lightweight ViT, and it processes the decoder input through transformer blocks. Lastly, the last layer of the decoder linearly projects output patches to pixels, and the pixel output is reshaped to form a reconstruction of the original image. The objective function is the mean squared error between the reconstruction and the original image. MAE has thrived because of its simple design and strong empirical performance.

In the main text, inspired by a follow-up work of MAE \cite{hu2022exploring}, we study the effect of masking by decoupling the patch size for masking images and the patch size hyperparameter in the ViT. Particularly, in the main text, we only vary the masking patch size and fix the ViT patch size at 8. Nevertheless, the original MAE \cite{he2021masked} does not decouple the two patch sizes. Therefore, for the reference of readers, in Appendix, we provide some analysis and results produced based on the patch size design from the original MAE \cite{he2021masked}, where the masking patch size and the ViT patch size are equal. We study three patch sizes: $\{8, 16, 32\}$. The experimental setup in \cite{hu2022exploring} and the setup in \cite{he2021masked} are interchangeable except for whether the patch size for the Vision Transformer varies.

\subsection{Pretraining and Linear Probing}
For pretraining MAE under different masking ratios or patch sizes, we leverage the Tensor Processing Unit (TPU) from Google Cloud. We train separate MAE models for each $($masking ratio, patch size$)$ pair, and each pretrained MAE corresponds to a unique masking ratio and patch size. We train all MAEs for $800$ epochs. Training time varies, with the shortest (patch size $= 32$) taking 18 hours on a TPU v3-128 Pod, and the longest (patch size $= 8$) taking 40 hours on a TPU v3-128 pod. The architecture follows the exact implementation from the original MAE paper \cite{he2021masked}, without any hyper-parameter tuning except masking ratio and patch size, which we study in this paper. Details of augmentation, initialization, and base learning rate scaling can be found in the Appendix section of \cite{he2021masked}, all of which we follow.

After pretraining, we also follow the original MAE work to use linear probing to evaluate the representation quality. After pretraining, we remove the projection layers and add a supervised learning classifier on frozen features of MAE encoders. The decoders are discarded during linear probing. Other details of linear probing can be found in the Appendix section of \cite{he2021masked}. We use the same hyper-parameters of linear probing as in \cite{he2021masked}.

\begin{figure*}[h]
\centering
\includegraphics[width=0.8\linewidth]{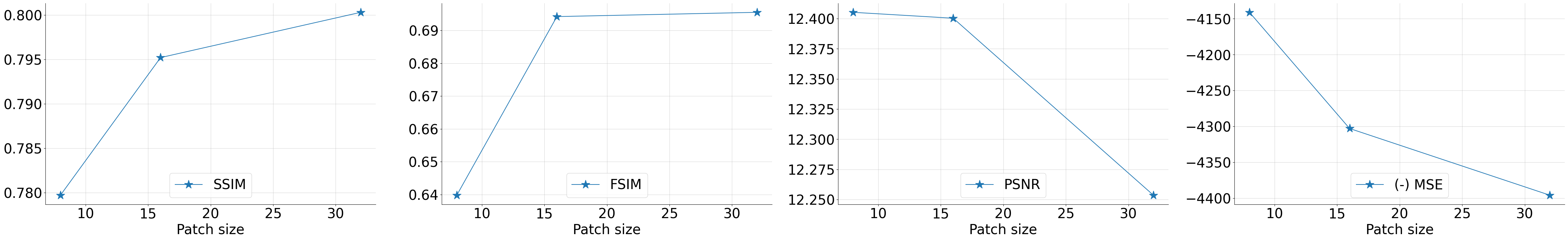}
\vspace{-.3em}
\caption{\textbf{Reconstruction evaluation} using the validation set without masking, based on two structural-level similarity metrics (SSIM and FSIM) and two pixel-level metrics (PSNR and MSE). We plot negative MSE for easier visualization. Higher SSIM and FSIM indicate high-level information is better captured, while higher PSNR and negative MSE indicates better low-level reconstruction. Here the patch size refers to the patch size in the original MAE, where the masking patch size and the patch size of ViT are equal.}
\vspace{-1em}
\label{fig:metrics_patch_size}
\end{figure*}

\subsection{Reconstructing high-level or low-level representations}
To perform reconstruction, we use both the encoder and the decoder from the pretrained MAEs. All samples from ImageNet-1K are passed through the encoder \textit{without} any masking, and the decoder reconstructs images in the original input space. Since no masking is applied, no masking token is applied to the input of the decoder. We use the reconstructed images and the original images to perform evaluations of four metrics: SSIM, FSIM, MSE, and PSNR. No training is performed, and the weights of the encoder and the decoder are frozen.

In Fig. \ref{fig:metrics_patch_size}, we show the reconstruction analysis using the original patch size design in MAE. Similar to the result in the main text, higher patch sizes
produce image reconstructions capturing high-level similarities better, while low patch sizes have reconstructions better on low-level metrics.

\subsection{Attention Analysis}
We follow the attention heatmap visualization in DINO \cite{caron2021emerging}, where the chosen token is the \texttt{[CLS]} token or an object-related token. We visualize the self-attention module from the last block of the MAE encoder ViT. Brighter colors suggest larger attention weights. For easier visualization, attentions below a threshold of activation scores are not shown. We use the same threshold as \cite{caron2021emerging}. For the self-attention visualization on the \texttt{[CLS]} token, we use an average of all heads in the last layer of the encoder ViT. For the self-attention visualization of the object-related token, we use the first head of the last layer of the encoder ViT, because using the average attention over all heads will result in a heatmap with much higher overall attention scores across pixels, making the visualization hard to interpret.

\begin{figure}[h]
\centering
\includegraphics[width=\linewidth]{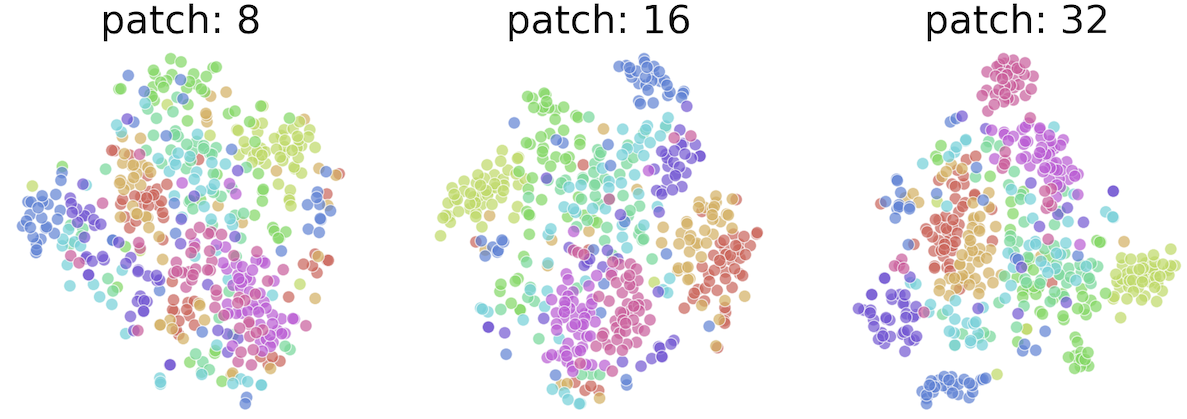}
\vspace{-.3em}
\caption{T-SNE embeddings of different MAE models under varied masking ratios and patch sizes. We fix the masking ratio at $0.75$ to change patch sizes. Each color represents one ImageNet class. The patch size refers to the patch size in the original MAE, where the masking patch size and the patch size of ViT are equal.}
\label{fig:tsne_patch_size}\vspace{.5em}
\end{figure}

\subsection{Linear separability}
To illustrate the linear separability of different MAEs under varied masking ratios or patch sizes, we sample ten random classes from ImageNet, and then use each MAE encoder to process images in the 10 classes to produce embeddings. We then project embeddings of all samples using PCA to a $50$-dimension space before t-SNE, as recommended by \cite{van2014accelerating}. For t-SNE, we use a perplexity of $20$.

In Fig. \ref{fig:tsne_patch_size}, we show the t-SNE plot using the original patch size design in MAE. Similar to the main text, embeddings are more separated in patch sizes 16 and 32 than 8, but differently, there are no significant differences between 16 and 32. Larger patch sizes generate more linearly separable embeddings in this case, although the separability seems indistinguishable for sizes 16 and 32.

For the robustness evaluation, we evaluate different variants of ImageNet validation datasets: ImageNet-v2 (INV2) \cite{shankar2020evaluating}, ImageNet-Adversarial (IN-A) \cite{hendrycks2021natural}, ImageNet-Rendition  \cite{barbu2019objectnet}, and ImageNet-Sketch (IN-S) \cite{wang2019learning}. We also include another object classification dataset,  ObjectNet (OJN) \cite{barbu2019objectnet}. ImageNet-v2 contains three new test sets with 10,000 new images each, sampled a decade after the collection of the original ImageNet dataset, and is independent of existing models to prevent overfitting. ImageNet-Adversarial consists of natural images with adversarial filtration, meaning samples that can be classified with spurious cues are removed.  Examples in ImageNet-A are harder to classify correctly and can cause mistakes across various models. ImageNet-Rendition contains renditions of ImageNet classes, such as art, cartoons, graffiti, and paintings. These examples share the same high-level object labels as ImageNet examples but differ in style and texture. ImageNet-Sketch contains black and white images of ImageNet classes, also differing in color and texture compared to original ImageNet samples. ObjectNet is a set of images captured at unusual poses in cluttered, natural scenes, which can severely degrade recognition performance. 

Note that for evaluating these datasets, no training is performed; we use the MAE encoders \textit{after} linear probings, therefore the checkpoints that are pretrained and linear-probed on ImageNet, and evaluate the checkpoints on these \textit{validation} datasets without any parameter updates.

In Table \ref{tab:robustness_patch_size}, we show the robustness analysis using the original patch size design in MAE. A moderate patch size $16$ yields the best robustness evaluation on IN-v2, OJN, IN-R, and IN-S. 
If we follow the original MAE and do not decouple masking patch size and ViT patch size, a medium patch size has stronger robustness performances than extreme patch sizes.

\begin{table}
\vspace{-.5em}
\tablestyle{3pt}{1.1}
\begin{tabular}{llllllll}
\multirow{1}{*}{mask ratio} &
\multirow{1}{*}{patch size}
& \multicolumn{1}{c}{IN1K} & \multicolumn{1}{c}{IN-v2} & \multicolumn{1}{c}{OJN} & \multicolumn{1}{c}{IN-R} & \multicolumn{1}{c}{IN-A}  & \multicolumn{1}{c}{IN-S} \\
\shline
0.75 & 8  & 62.57 & 49.17 & 13.44 & 19.42 & 3.73 &  10.73 \\
0.75 & 16 & 67.41 & 54.23 & 18.24 & 25.20 & 3.76 & 15.51 \\
0.75 & 32 & 55.51 & 42.35 & 13.46 & 18.70 & 1.89 & 9.48 \\
\end{tabular}
\vspace{-.8em}
\caption{\textbf{Accuracy ($\%$) of linear probing and robustness evaluation} on ImageNet variants and ObjectNet. We linear probe MAE via supervised training on IN1K, and then perform inference on IN1K as well as other evaluation sets.  We fix the masking ratio at $0.75$ to change patch sizes. The patch size refers to the patch size in the original MAE, where the masking patch size and the patch size of ViT are equal.}
\label{tab:robustness_patch_size}
\end{table}

\subsection{Shape bias}
The cue-conflict dataset was introduced by \cite{geirhos2018imagenet} to evaluate how much deep learning models rely on shape information for prediction, which reflects the model's robustness to spurious correlation like textures.
This dataset consists of 1280 images synthesized from 160 images of objects and 48 images of textures.
The shape accuracy is measured by the fraction of images predicted correctly by their shape.
We directly run the pretrained MAE models with linear probes trained on ImageNet-1K on the cue-conflict dataset to examine the representation resulting from MAE pretraining without any adaptation to the test dataset.


\subsection{Transfer learning}
\begin{table}[t]
\tablestyle{5pt}{1.05}
\begin{tabular}{llcc}
mask ratio & patch size & AP$^\text{box}$ & {AP$^\text{mask}$}\\ 
\shline
0.75 &  8 & 34.21 & 32.28 \\
0.75 &  16 & 33.77 & 32.04 \\
0.75 &  32 & 32.39 & 30.54 \\

\end{tabular}
\caption{\textbf{COCO object detection and segmentation} using a ViT Mask R-CNN baseline. We fix the masking ratio at $0.75$ to change patch sizes. The patch size refers to the patch size in the original MAE, where the masking patch size and the patch size of ViT are equal.}
\label{tab:coco_patch} \vspace{-1em}
\end{table}
We use the pretrained MAE ViT encoder as an FPN \cite{lin2017feature} backbone in Mask-RCNN \cite{he2017mask}, following \cite{he2021masked}. To do so, \cite{he2021masked} uses a stack of pretrained transformer blocks in MAE to produce feature maps at a single scale; for instance, patch size $16$ will produce stride $16$ features. Then the features are equally divided, and upsampling or downsampling is applied to create features at different scales. Lastly, the FPN is built on multi-scale features. Below we include the transfer learning results of different patch sizes on COCO object detection and segmentation \cite{lin2014microsoft}. Because different patch sizes in ViT will influence the scale of feature maps in the FPN, we enforce the same combinations of multi-scale features: i.e., stride $4$, $8$, $16$, and $32$. 

From Table \ref{tab:coco_patch}, we show the transfer learning results of MAE under different patch sizes. Patch size $8$ performs the best, and patch size $16$ is better than $32$. The reason for the better performance at patch size $8$ may be due to a smaller batch size used, compared to patch size $16$ and $32$ (we can only fit batch size $1$ for patch size $8$ due to the increased number of tokens to process because of a smaller patch size.) We use the same batch size for $32$ and $16$, and the comparison between the two supports our claim: an extreme masking scheme can hurt the model's capacity to capture high-level information or, in this case, the semantic understanding of the scene.